\documentclass[10pt,journal,cspaper,compsoc]{IEEEtran}

\usepackage[nocompress]{cite}
\usepackage[pdftex]{graphicx}
\usepackage[cmex10]{amsmath}
\usepackage{amssymb,url,xspace}
\usepackage{booktabs,balance}
\usepackage[ruled,vlined]{algorithm2e}
\usepackage{array}
\usepackage{subcaption}
\usepackage[bookmarks,colorlinks,linkcolor=blue,citecolor=blue]{hyperref}
\usepackage{fixltx2e}

\graphicspath{{./}{./figures/}}
\DeclareGraphicsExtensions{.pdf,.jpeg,.png,.jpg}

\makeatletter
\DeclareRobustCommand\onedot{\futurelet\@let@token\@onedot}
\def\@onedot{\ifx\@let@token.\else.\null\fi\xspace}

\def\eg{\emph{e.g}\onedot} 
\def\ie{\emph{i.e}\onedot} 
\def\cf{\emph{c.f}\onedot}

\def\etal{\emph{et al}\onedot}
\makeatother

\def\Vec#1{{\boldsymbol{#1}}}
\def\Mat#1{{\boldsymbol{#1}}}
\def\SPD#1{\mathcal{S}_{++}^{#1}}

\newcommand{\tr}{\mathop{\rm  Tr}\nolimits}

\usepackage{amsthm} 
 
\newtheorem{theorem}{Theorem}
\newtheorem{lemma}{Lemma} 
\newtheorem{proposition}{Proposition}

\newtheorem{definition}{Definition}

\begin{document}

\title{Sparse Coding on Symmetric Positive Definite Manifolds using Bregman Divergences}

\author
  {
  Mehrtash~Harandi,~
  Richard~Hartley,~
  Brian~Lovell,~
  and~Conrad~Sanderson
  \thanks
    {
    Mehrtash~Harandi and Richard~Hartley are with the College of Engineering and Computer Science, Australian National University,
    and NICTA, Australia.
    Brian~Lovell is with the University of Queensland, Australia.
    Conrad~Sanderson is with NICTA, Australia, and the Queensland University of Technology, Australia.
    }
  \thanks{Contact e-mail: mehrtash.harandi@nicta.com.au}
  \thanks{Acknowledgements: NICTA is funded by the Australian Government as represented by the Department of Broadband, Communications and the Digital Economy and the Australian Research Council (ARC) through the ICT Centre of Excellence program.}
  }

\markboth{IEEE Transactions on Neural Networks and Learning Systems}%
{Harandi \etal : Sparse Coding on Symmetric Positive Definite Manifolds}

\IEEEcompsoctitleabstractindextext{
\begin{abstract}

This paper introduces sparse coding and dictionary learning for Symmetric Positive Definite (SPD) matrices,
which are often used in machine learning, computer vision and related areas.
Unlike traditional sparse coding schemes that work in vector spaces,
in this paper we discuss how SPD matrices can be described by sparse combination of dictionary atoms,
where the atoms are also SPD matrices.
We propose to seek sparse coding by embedding the space of SPD matrices into Hilbert spaces through two types of Bregman matrix divergences.
This not only leads to an efficient way of performing sparse coding, but also an online and iterative scheme for dictionary learning.
We apply the proposed methods to several computer vision tasks where images are represented by region covariance matrices.
Our proposed algorithms outperform  state-of-the-art methods on a wide range of classification tasks,
including face recognition, action recognition, material classification and texture categorization.

\end{abstract}

}

\maketitle
\IEEEdisplaynotcompsoctitleabstractindextext

\IEEEpeerreviewmaketitle


\section{Introduction}
\label{sec:intro}

Sparsity is a popular concept in signal processing~\cite{ELAD_SR_BOOK_2010,Olshausen_1996_Nature,Wright_2009_PAMI}
and stipulates that natural signals like images can be efficiently described using only a few non-zero coefficients of a suitable 
basis (\ie dictionary)~\cite{ELAD_SR_BOOK_2010}. This paper introduces techniques to perform sparse coding on Symmetric Positive Definite (SPD) 
matrices. More specifically,  unlike traditional sparse coding schemes that work on vectors, 
in this paper we discuss how SPD matrices can be described by sparse combination of dictionary atoms,
where the atoms are also  SPD matrices.

Our motivation stems from pervasive role of SPD matrices in machine learning, computer vision and related areas. 
For example, SPD matrices have been used in medical imaging,
texture classification~\cite{Tuzel_ECCV_2006,Harandi_ECCV_2012},
action recognition and gesture categorization~\cite{Sanin_WACV_2013},
as well as face recognition~\cite{Pang_TCSVT_2008,Harandi_ECCV_2012}.

Extending sparse coding methods to SPD matrices is not trivial, since such matrices form the interior of the positive semidefinite cone.
In other words, simply vectorizing SPD matrices and employing Euclidean geometry (\eg, Euclidean norms) does not lead to accurate representations~\cite{Pennec_IJCV_2006,Tuzel_PAMI_2008,Sadeep_CVPR_2013}.
To overcome the drawbacks of Euclidean structure, SPD matrices are usually analyzed using
a Riemannian structure, known as SPD or tensor manifold~\cite{Pennec_IJCV_2006}.
This is where the difficulties arise. On one hand, taking into account the Riemannian geometry is important as discussed in 
various recent studies~\cite{Pennec_IJCV_2006,Tuzel_PAMI_2008,Harandi_ECCV_2012,Sadeep_CVPR_2013}. On the other hand,
the non-linearity of the Riemannian structure is a hindrance and demands specialized machineries.

Generally speaking, two approaches to handle the non-linearity of Riemannian manifolds are
(i)~locally flattening them via tangent spaces~\cite{Tuzel_PAMI_2008,Sanin_WACV_2013},
and
(ii)~embedding them in higher dimensional Hilbert spaces~\cite{Harandi_ECCV_2012,Caseiro_ECCV_2012,Sadeep_CVPR_2013}. 
The latter has recently received a surge of attention, since embedding into Reproducing Kernel Hilbert Space (RKHS) through kernel 
methods~\cite{Shawe-Taylor2004book}
is a well-established and principled approach in machine learning. 
However, embedding SPD manifolds into RKHS requires \mbox{non-trivial} kernel functions defined on such manifolds,
which, according to Mercer's theorem~\cite{Shawe-Taylor2004book}, must be positive definite.

The contributions in this paper%
\footnote{%
  This paper is a thoroughly extended and revised version of our earlier work~\cite{Harandi_ECCV_2012}. 
  In addition to providing more insights on the proposed methods,
  we extend our primary ideas by studying and devising coding and dictionary learning methods in the RKHS induced by the Jeffrey kernel.
  We also devise an efficient algorithm to obtain sparse codes in our RKHS-based formulation.}
are four-fold: 
\begin{itemize}

\item[(i)]
We propose sparse coding and dictionary learning algorithms for data points (matrices) on SPD manifolds, by embedding the 
manifolds into RKHS. This is advantageous, as linear geometry applies in RKHS.

\item[(ii)]
For the embedding we propose kernels derived from two Bregman matrix divergences, namely the Stein and Jeffrey divergences.
While the kernel property of the Jeffrey divergence was discovered in 2005~\cite{Hein_2005}, to our best knowledge, 
this is one of the first attempts to benefit from this kernel for analyzing SPD matrices.

\item[(iii)]
For both kernels, we devise a closed-form solution for updating an SPD dictionary atom by atom.

\item[(iv)]
We apply the proposed methods to
several computer vision tasks where images are represented by region covariance matrices.  
Our proposed algorithms outperform  state-of-the-art methods
on several classification tasks,
including face recognition, texture classification and action recognition.
\end{itemize}

\section{Related Work}
\label{sec:related_work}

In computer vision, SPD matrices are used in various applications, including  
pedestrian detection~\cite{Tuzel_PAMI_2008},
texture classification~\cite{Tuzel_ECCV_2006,Harandi_ECCV_2012},
object recognition~\cite{Sadeep_CVPR_2013},
object tracking~\cite{Tuzel_ECCV_2006},
action recognition~\cite{Guo_TIP13,Sanin_WACV_2013}
and
face recognition~\cite{Pang_TCSVT_2008,Harandi_ECCV_2012}.
This is mainly because Region Covariance Descriptors (RCM)~\cite{Tuzel_ECCV_2006}, which encode second order statistics, 
are straightforward and relatively robust descriptors for images and videos.
Moreover, structure tensors, which are by nature SPD matrices, encode important image features
(\eg, texture and motion in optical flow estimation and motion segmentation).
Lastly, diffusion tensors that naturally arise in medical imaging are described by $3\times3$ SPD matrices~\cite{Pennec_IJCV_2006}.

Our interest in this paper is to perform sparse coding and dictionary learning on SPD matrices,
since modern systems in various applications benefit from the notion of sparse coding.
However, while significant steps have been taken to develop the theory of the sparse coding and dictionary learning in Euclidean spaces,
only a handful of studies tackle similar problems for SPD matrices~\cite{Sra_ECML_2011,Guo_TIP13,Vemuri_ICML_2013}.

Sra and Cherian~\cite{Sra_ECML_2011}  proposed to measure the similarity between SPD matrices using the Frobenius norm 
and formulated the sparse coding and dictionary learning problems accordingly.
While solving the problems using purely Euclidean structure of SPD matrices is computationally attractive,
it neglects the Riemannian structure of SPD manifolds.

A somehow similar and straightforward idea is to flatten an SPD manifold using a fixed tangent space.
Sparse coding by embedding manifolds into their identity tangent spaces, which identifies the Lie algebra of SPD manifolds, is considered in~\cite{Yuan_ACCV_2010,Guo_TIP13,Faraki_2015}.
Though such embedding considerably simplifies the sparse coding formulation, the pair-wise distances are no longer adequate, which can affect discrimination performance.
This is exacerbated for manifolds with negative curvature (\eg SPD manifolds), since pair-wise distances are not even directly bounded%
\footnote{For manifolds with positive curvature,
pair-wise distances on tangent spaces are greater or equal to true geodesic distances on the manifold according to Toponogov's theorem~\cite{Lee_Manifold}.
Such property does not hold for manifolds with negative curvature.}.

A more involved approach to learn a Riemannian dictionary is proposed very recently by Ho \etal~\cite{Vemuri_ICML_2013}.
The underlying idea is to exploit the tangent bundle of the manifold.
To avoid a trivial solution in this approach, an affine constraint has to be added to the general formulation~\cite{Vemuri_ICML_2013}. 
While this results in independency to the origin, it no longer addresses the original problem. 
Furthermore, switching back and forth to tangent spaces of SPD manifolds 
(as required by this formulation) can be computationally very demanding for high dimensional manifolds.

Sivalingam \etal~\cite{Sivalingam_ICCV_2011,TSC_PAMI_2014} proposed Tensor Sparse Coding (TSC) 
which utilizes the Burg divergence (an asymmetric type of Bregman divergence)
to perform sparse coding and dictionary learning on SPD manifolds.
To this end, they show that when the Burg divergence is used as the proximity measure, the problem of sparse coding becomes a MAXDET problem which is convex
and hence can be solved by interior point algorithms~\cite{Sivalingam_ICCV_2011}.
As for dictionary learning, two methods were proposed in~\cite{Sivalingam_ICCV_2011,TSC_PAMI_2014}.
In the first method, a gradient descent approach was utilized to update dictionary atoms one by one.
Inspired by the K-SVD algorithm~\cite{Aharon_2006},
the second method updates dictionary atoms by minimizing a form of residual error over training data,
which speeds up the process of dictionary learning.
Besides the asymmetric nature of the Burg divergence, we note that the computational complexity of the TSC algorithm is high, especially for high-dimensional SPD manifolds.  

\section{Preliminaries}
\label{sec:preliminaries}

This section provides an overview on Riemannian geometry of SPD manifolds, Bregman divergences and their properties.
It provides  the groundwork for techniques described in following sections. 
Throughout the paper, bold capital letters denote matrices (\eg, $\Mat{X}$) and bold lower-case letters denote column vectors (\eg, $\Vec{x}$).
Notation $x_i$ is used to indicate element at position $i$ of vector $\Vec{x}$.
$\mathbf{I}_n$ is the $n \times n$ identity matrix.
$\Vert \Vec{x} \Vert_2 = \sqrt{\Vec{x}^T\Vec{x}}$ and $\Vert \Vec{x} \Vert_1 = \sum\nolimits_i|x_i|$
denote the $\ell_2$ and $\ell_1$ norms, respectively, with $T$ indicating matrix transpose.
$\Vert \Mat{X} \Vert_F = \sqrt{\tr \big(\Mat{X}^T\Mat{X}\big)}$ designates the Frobenius norm.
$\mathrm{GL}(n)$ denotes the  general linear group, the group of real invertible $n \times n$ matrices.
$\mathrm{Sym}(n)$ is the space of real $n \times n$ symmetric matrices.
%

\subsection{Riemannian Geometry of SPD Manifolds}

An $n \times n$, real SPD matrix $\Mat{X}$ has the property that $\Vec{v}^T\Mat{X}\Vec{v} > 0$ for all non-zero $\Vec{v} \in \mathbb{R}^n$. 
The space of $n \times n$ SPD matrices, denoted by $\SPD{n}$,
is not a vector space since multiplying an SPD matrix by a negative scalar results in a matrix which does not belong to $\SPD{n}$.
Instead, $\SPD{n}$ forms the interior of a convex cone in the $n(n+1)/2$-dimensional Euclidean space.
The $\SPD{n}$ space is mostly studied when endowed with a Riemannian metric and thus forms a Riemannian manifold~\cite{Pennec_IJCV_2006}. 

On a Riemannian manifold, a natural way to measure nearness is through the notion of geodesics, 
which are curves analogous to straight lines in $\mathbb{R}^n$.
The geodesic distance is thus defined as the length of the shortest curve connecting the two points.
The tangent space at a point {\small $\Mat{P}$} on the manifold, $T_{\Mat{P}}{\mathcal{M}}$, is a vector space that consists of the
tangent (\ie, velocity) vectors of all possible curves passing through {$\Mat{P}$}.

Two operators, namely the exponential map {$\exp_{\Mat{P}}(\cdot):T_{\Mat{P}}{\mathcal{M}} \rightarrow \mathcal{M}$}
and the logarithm map {$\log_{\Mat{P}}(\cdot)=\exp^{-1}_{\Mat{P}}(\cdot):\mathcal{M} \rightarrow T_{\Mat{P}}{\mathcal{M}}$},
are defined over Riemannian manifolds to switch between the manifold and tangent space at {\small $\Mat{P}$}.
The exponential operator maps a tangent vector {$\Delta$} to a point {$\Mat{X}$} on the manifold.
The property of the exponential map ensures that the length of {$\Delta$}
becomes equal to the geodesic distance between {$\Mat{X}$} and {$\Mat{P}$}.
The logarithm map is the inverse of the exponential map,
and maps a point on the manifold to the tangent space {$T_{\Mat{P}}$}.
The exponential and logarithm maps vary as point {$\Mat{P}$} moves along the manifold.

On the SPD manifold, the Affine Invariant Riemannian Metric (AIRM)~\cite{Pennec_IJCV_2006},
defined as:
\begin{align}
	\langle \Mat{V}, \Mat{W} \rangle_\Mat{P} &\triangleq  \langle \Mat{P}^{-1/2}\Mat{V}\Mat{P}^{-1/2}, \Mat{P}^{-1/2}\Mat{W}\Mat{P}^{-1/2} \rangle
	\notag \\
	&= \tr \left( \Mat{P}^{-1} \Mat{V} \Mat{P}^{-1} \Mat{W}\right)\;,
	\label{eqn:AIRM_equ}
\end{align}

\noindent
for $\Mat{P} \in \SPD{n}$ and $\Mat{V},\Mat{W} \in T_{\Mat{P}}{\mathcal{M}}$, induces the following geodesic distance between points
$\Mat{X}$ and $\Mat{Y}$:

\begin{equation}
\delta_R(\Mat{X},\Mat{Y}) = \|\log(\Mat{X}^{-1/2}\Mat{Y}\Mat{X}^{-1/2})\|_F\;,
\label{eqn:geodesic_distance}
\end{equation}
with $\log(\cdot)$ being the principal matrix logarithm.
%

\subsection{Bregman Divergences}
\label{sec:subsec_bregman}

In this part we introduce two divergences derived from Bregman matrix divergence, namely the Jeffrey and Stein divergences.
We discuss their properties and establish their relations to AIRM.
This provides motivation and grounding for our formulation of sparse coding and dictionary learning using the aforementioned divergences.

\begin{definition} \label{def:bregman_divergence}
		Let $\zeta : \mathcal{S}_{++}^{n} \rightarrow \mathbb{R}$ be a strictly
		convex and differentiable function defined on the symmetric positive cone $\SPD{n}$. The Bregman 
		matrix divergence 	$d_\zeta : \SPD{n} \times \SPD{n} \rightarrow [0,\infty)$ is defined as
		\begin{equation}
			\label{eqn:Bregman_Div}
			d_\zeta(\Mat{X},\Mat{Y}) = \zeta(\Mat{X}) - \zeta(\Mat{Y}) - \langle \nabla_{\zeta} (\Mat{Y}) , \Mat{X} - \Mat{Y} \rangle \; ,
		\end{equation}
		where
		\mbox{\small{$\langle \Mat{X} , \Mat{Y} \rangle \mbox{=} \tr \left( \Mat{X}^T \Mat{Y} \right) $}},
		and	$ \nabla_{\zeta}(\Mat{Y})$	 represents the gradient of $\zeta$ evaluated at $\Mat{Y}$.
\end{definition}

Loosely speaking, the Bregman divergence between $\Mat{X}$ and $\Mat{Y}$ can be understood as 
the distance between the function $\zeta(\Mat{X})$ and its first order Taylor approximation constructed at $\Mat{Y}$. 	
The Bregman divergence is asymmetric, non-negative, and definite (\ie , $d_\zeta(\Mat{X},\Mat{Y}) = 0, \; \text{iff}~ \Mat{X} =\Mat{Y}$).
While the Bregman divergence enjoys a variety of useful properties~\cite{Kulis:2009:JMLR},
its asymmetric behavior can be a hindrance (\eg, in SVMs, the kernels need to be symmetric, hence asymmetric divergences 
cannot be used to devise kernels).
In this paper we are interested in two types of symmetrized Bregman divergences, namely the {\it Jeffrey}  and the {\it Stein} divergences.
	 
\begin{definition} \label{def:kl_divergence}
	The $J$~divergence (also known as Jeffrey or symmetric KL divergence) is obtained from the Bregman divergence of 	
	Eqn.~\eqref{eqn:Bregman_Div} by using $\zeta(\Mat{X}) = -\log |\Mat{X} |$ as the seed function where $|\cdot|$ denotes determinant:
	\begin{align}
		J(\Mat{X},\Mat{Y}) &\triangleq \frac{1}{2} d_\zeta(\Mat{X},\Mat{Y}) + \frac{1}{2}d_\zeta(\Mat{Y},\Mat{X}) \notag\\
		&= \frac{1}{2}\tr (\Mat{X}^{-1}\Mat{Y}) - \frac{1}{2}\log |\Mat{X}^{-1}\Mat{Y}| \notag\\
		&+   \frac{1}{2}\tr (\Mat{Y}^{-1}\Mat{X}) - \frac{1}{2}\log |\Mat{Y}^{-1}\Mat{X}| -n \notag\\			  
		&= \frac{1}{2}\tr (\Mat{X}^{-1}\Mat{Y}) + \frac{1}{2}\tr (\Mat{Y}^{-1}\Mat{X}) -n \;.
		\label{eqn:KL_Div}
	\end{align}

\end{definition}
	
\begin{definition} \label{def:stein_divergence}
	The Stein or $S$~divergence (also known as Jensen-Bregman LogDet divergence~\cite{Cherian_PAMI12})
	 is obtained from the Bregman divergence of Eqn.~\eqref{eqn:Bregman_Div}
	by again using $\zeta(\Mat{X}) = - \log |\Mat{X} |$ as the seed function but through {\it Jensen-Shannon} symmetrization:
	\begin{align}
	   	S(\Mat{X},\Mat{Y}) &\triangleq \frac{1}{2}	
   		d_{\zeta}\left( \Mat{X},\frac{\Mat{X}+\Mat{Y}}{2} \right) +
	   	\frac{1}{2} d_{\zeta}\left( \Mat{Y},\frac{\Mat{X}+\Mat{Y}}{2} \right) \notag\\
    	&= \log \bigg| \frac{\Mat{X}+\Mat{Y}}{2} \bigg|
    	- \frac{1}{2}  \log | \Mat{X}\Mat{Y} |  \;. 
    	\label{eqn:Stein_Div}
	\end{align}%
\end{definition}

\subsection{Properties of $J$ and $S$ divergences}
\label{sec:subsec_bregman_prop}

The $J$ and $S$ divergences have a variety of properties which are akin to those of AIRM.
The pertinent properties which inspired us to seek sparse coding on $\SPD{n}$ using such divergences are:

\begin{itemize}

\item
Both the J and S divergences as well as AIRM (as its name implies) are invariant to affine transformation~\cite{Pennec_IJCV_2006,Sra_NIPS_2012,Vemuri_CVPR2004}.

\item
the length of curves under AIRM and S divergence is equal up to a scale~\cite{Harandi_ECCV_2014}.

\item
the geometric mean of two tensors under AIRM coincides with the geometric mean under J and S divergences
(see~\cite{Sra_NIPS_2012} for the S divergence and the appendix for the proof on the J divergence). 

\end{itemize} 

The key message worth noting is the Hilbert space embedding property of the $J$ and $S$ divergences,
which does not hold for AIRM~\cite{Harandi_ECCV_2012,Sadeep_CVPR_2013}.

\subsubsection*{\textbf{Hilbert space embedding (SPD kernels)}}

Both $J$ and $S$ divergences admit a Hilbert space embedding in the form of a Radial Basis Function (RBF) kernel~\cite{Shawe-Taylor2004book}. 
More specifically, for the $J$-divergence it has been shown that the kernel 
\begin{equation}
	k_J(\Mat{X},\Mat{Y}) = \exp \{-\beta J(\Mat{X},\Mat{Y}) \},
	\label{eqn:kernel_j_div}
\end{equation}
is conditionally positive definite~\cite{Hein_2005}. Formally:
\begin{definition}[Conditionally Positive Definite Kernels]
Let $\mathcal{X}$ be a nonempty set. A symmetric function $\psi: \mathcal{X} \times \mathcal{X} \to \mathbb{R}$ is a conditionally
positive definite kernel on $\mathcal{X}$ if and only if
$\sum_{i,j=1}^nc_ic_jk(x_i,x_j) \geq 0$ for any $n \in \mathbb{N}$, $x_i \in \mathcal{X}$ and $c_i \in \mathbb{R}$ with $\sum_{i=1}^n c_i = 0$.
\end{definition}

The relations between positive definite (\emph{pd}) and conditionally positive definite (\emph{cpd}) kernels
are studied by Berg \etal~\cite{Berg:1984} and Sch\"{o}lkopf~\cite{Scholkopf:NIPS:2001}.
An important property of \emph{cpd} kernels is 
\begin{proposition}
For a kernel algorithm that is translation invariant, \emph{cpd} kernels can be used instead of \emph{pd} kernels~\cite{Scholkopf:NIPS:2001}.
\end{proposition}

This property relaxes the requirement of having \emph{pd} kernels for certain types of kernel algorithms. 
For example, in SVMs, a \emph{cpd} kernel can be seamlessly used instead of a \emph{pd} kernel.
We note that in~\cite{Moreno_2003} the kernel $k_J(\cdot,\cdot)$ was claimed to be positive definite. 
However, a formal proof is not available according to our best knowledge. 
For the Stein divergence, the kernel 
\begin{equation}
	k_S(\Mat{X},\Mat{Y}) = \exp \{ -\beta S(\Mat{X},\Mat{Y}) \},
	\label{eqn:kernel_s_div}
\end{equation}
is guaranteed to be positive definite for 
\begin{equation}
    \beta \in \left \{ \frac{1}{2},\frac{2}{2}, \cdots, \frac{n-1}{2} \right \}
    \cup \left \{\tau \in \mathbb{R}: \tau > \frac{1}{2}(n-1) \right \}\;.
    \label{eqn:Stein_Krnl2}
\end{equation}%
Interested reader is referred to~\cite{Sra_NIPS_2012} for further details. For values of $\beta$ outside of the above set,
it is possible to convert a pseudo kernel into a true kernel,
as discussed for example in~\cite{Similarity_JMLR_2009}.

\section{Sparse Coding}
\label{sec:sparse_coding}

Given a query $\Vec{x} \in \mathbb{R}^d$,
sparse coding in vector spaces optimizes the objective function
\begin{equation}
	l_E(\Vec{x},\mathbb{D}) \triangleq 
    \underset{\Vec{y}}{\min} \:
    \Bigl\| \Vec{x}- \sum\nolimits_{j=1}^{N} y_j \Vec{d}_j
    \Bigr\|_2^2
    + \rm{Sp}( \Vec{y} ),
    \label{eqn:euc_sparse_coding}
\end{equation}
with \mbox{\small$\mathbb{D}_{d \times N} = \left[ \Vec{d}_1 | \Vec{d}_2| \cdots | \Vec{d}_N \right],\;\Vec{d}_i \in \mathbb{R}^d,\; N>d$}
being a dictionary of size $N$. The function $\rm{Sp}( \Vec{y} )$ penalizes the solution if it is not sparse.
The most common form of $l_E(\Vec{x},\mathbb{D})$ in the literature is obtained via $\ell_1$-norm regularization:
\noindent
\begin{equation}
	l_E(\Vec{x},\mathbb{D}) \triangleq 
    \underset{\Vec{y}}{\min} \:
    \Bigl\| \Vec{x}- \sum\nolimits_{j=1}^{N} y_j \Vec{d}_j
    \Bigr\|_2^2
    +\lambda \|\Vec{y}\|_1. 
    \label{eqn:euc_sparse_coding2}
\end{equation}

As elaborated in~\cite{Harandi_ICCV_2013}, directly translating the sparse coding problem to a non-flat Riemannian manifold $\mathcal{M}$ 
with a metric {$\left\| \cdot \right\|_\mathcal{M}$} (such as geodesic distance)  leads to re-writing Eqn.~\eqref{eqn:euc_sparse_coding2} as:
\begin{equation}
	l_\mathcal{M}(\Mat{X},\mathbb{D}) \triangleq  
	\underset{\Vec{y}}{\min} \:  
    \Bigl\| \Mat{X} \ominus
    \biguplus\nolimits_{j=1}^{N} y_{j} \odot \Mat{D}_j \Bigr\|_\mathcal{M}^2
    +\lambda \|\Vec{y}\|_1,    
    \label{eqn:Riemannian_sc}
\end{equation}
where $\mathbb{D} = \big\{ \Mat{D}_i \big\}_{i=1}^{N},\; \Mat{D}_i \in \mathcal{M}$ is a Riemannian dictionary and 
$\Mat{X} \in \mathcal{M}$ is a query point. The operators $\ominus$, $\biguplus$ and $\odot$ are Riemannian replacements
for subtraction, summation and scalar multiplication, respectively. 
We note that the operators $\ominus$ and $\biguplus$ should be commutative and associative.

There are several difficulties in solving Eqn.~\eqref{eqn:Riemannian_sc}.
For example, metrics on Riemannian manifolds do not generally result in Eqn.~\eqref{eqn:Riemannian_sc} being convex~\cite{Harandi_ICCV_2013}.
As such, instead of solving Eqn.~\eqref{eqn:Riemannian_sc}, here we propose to side-step the difficulties by
embedding the manifold $\mathcal{M}$ into a Hilbert space $\mathcal{H}$
and replacing the idea of ``combination'' on manifolds with the general concept of linear combination in Hilbert spaces.

For the SPD manifold $\SPD{n}$, our idea is implemented as follows.
Let $\mathbb{D}=\{\boldsymbol{D}_1,\boldsymbol{D}_2,\cdots,\boldsymbol{D}_N\};{ }\boldsymbol{D}_i \in \SPD{n}$ 
and $\phi:\SPD{n} \rightarrow \mathcal{H}$ be 
a Riemannian dictionary and an embedding function on $\SPD{n}$, respectively.
Given a Riemannian point~{$\Mat{X}$}, we seek a sparse vector {$\Vec{y} \in \mathbb{R}^N$}
such that $\phi(\Mat{X})$ admits the sparse representation $\Vec{y}$
over $\{\phi(\Mat{D}_1),\phi(\Mat{D}_2),\cdots,\phi(\Mat{D}_N)\}$.
In other words, we are interested in solving the following problem:

\noindent
\begin{equation}
	l_\phi(\Mat{X},\mathbb{D}) \triangleq 
    \underset{\Vec{y}}{\min} \:
    \Big\| \phi\big(\Mat{X}\big)- \sum\limits_{j=1}^{N} y_j \phi\big(\Mat{D}_j\big)
    \Big\|_2^2
    +\lambda \|\Vec{y}\|_1. 
    \label{eqn:kernel_sparse_coding}
\end{equation}

For both $J$ and $S$ divergences,
an embedding $\phi$ with a reproducing kernel property~\cite{Shawe-Taylor2004book} exists as explained in~\textsection\ref{sec:preliminaries}.
This enables us to use the kernel property
$k(\Mat{X},\Mat{Y}) = \phi\big(\Mat{X}\big)^T\phi\big(\Mat{Y}\big)$
to expand the $\ell_2$ term in Eqn.~\eqref{eqn:kernel_sparse_coding} as:

\noindent
\begin{align}
    &~~~~\Big\| \phi\big(\Mat{X}\big)-\sum\limits_{j=1}^{N}y_j \phi\big(\Mat{D}_j\big) \Big\|_2^2 
    = \phi\big(\Mat{X}\big)^T\hspace{-1ex}\phi\big(\Mat{X}\big) \nonumber \\
    %
    &-2\hspace{-0.5ex}\sum\limits_{j=1}^{N}y_j \phi\big(\Mat{D}_j\big)^T \phi\big(\Mat{X}\big)
    +\sum\limits_{i,j=1}^{N} y_i y_j\phi\big(\Mat{D}_i\big)^T \phi\big(\Mat{D}_j\big)\nonumber \\  
    & = k(\Mat{X},\Mat{X})-2\Mat{y}^T\mathcal{K}(\Mat{X},\mathbb{D})+
    \Vec{y}^T \boldsymbol{\mathbb{K}(\mathbb{D},\mathbb{D})} \Vec{y},
    \label{eqn:KSR_Opt}
\end{align}%

\noindent
where
 $\mathcal{K}(\Mat{X},\mathbb{D})=[a_{i}]_{N \times 1}; ~ a_{i}=k(\boldsymbol{X},\boldsymbol{D}_i)$
and
$\mathbb{K}(\mathbb{D},\mathbb{D})=[a_{ij}]_{N \times N}$, with $ a_{ij}=k(\boldsymbol{D}_i,\boldsymbol{D}_j)$.
Since $k(\cdot,\cdot)$ is a reproducing kernel, $\mathbb{K}$ is positive definite. 
This reveals that the optimization problem in Eqn.~\eqref{eqn:KSR_Opt}
is convex and similar to its counterpart in Euclidean space,
except for the definition of $\mathcal{K}$ and $\mathbb{K}$.
Consequently, greedy or relaxation solutions can be adapted to obtain the sparse codes~\cite{ELAD_SR_BOOK_2010}.
To solve Eqn.~\eqref{eqn:KSR_Opt} efficiently, we have extended the Feature-Sign Search Algorithm (FSSA)~\cite{NIPS2006_NG} 
to its kernel version (kFSSA) in Appendix~\ref{app:kernel_feature_sign_alg}.

We note that kernel sparse coding and dictionary learning in traditional Euclidean spaces are studied recently 
in~\cite{Gao_ECCV_2010,Nguyen_TIP_2013}.
In contrast, our aim is to obtain sparse coding of points on SPD manifolds, using SPD matrices as dictionary atoms.
In our proposed solution this requires dedicated SPD kernels. Moreover, as will be discussed in~\textsection~\ref{sec:dic_learning} dedicated algorithms for dictionary learning should be devised.

~

\subsection{Classification Based on Sparse Representation}
\label{sec:sparse_classification}

If the atoms in the sparse dictionary are not labeled
(for example if $\mathbb{D}$  is a generic dictionary not tied to any particular class),
the generated sparse codes (vectors) for both training and query data can be fed to Euclidean-based classifiers
like support vector machines~\cite{Bishop_2006} for classification.
In a supervised classification scenario,
\ie,~if the atoms in sparse dictionary $\mathbb{D}$ are labeled,
the generated sparse codes of the query sample can be directly used for classification. Let
\mbox{$\Vec{y}_i = [ y_0\delta( l(0)-i ), ~y_1\delta( l(1)-i ), ~\cdots, ~y_N\delta( l(N)-i ) ]^T$}
be the class-specific sparse codes,
where {$l(j)$} is the class label of atom~$\Mat{D}_j$
and $\delta(x)$ is the discrete Dirac function~\cite{Bishop_2006}.
An efficient way of utilizing class-specific dictionary is through computing residual errors~\cite{Wright_2009_PAMI}.
In this case, the residual error of query sample $\Mat{X}$ for class $i$ is defined as:

\begin{equation}
    \varepsilon_i(\Mat{X})=  \Big\| \phi(\Mat{X}) - \sum\limits_{j=1}^{N} y_j \phi(\Mat{D}_j) \delta( l(j)-i ) \Big\|^2 .
    \label{eqn:sparse_classification1}
\end{equation}%

Expanding Eqn.~\eqref{eqn:sparse_classification1}
and noting that $k(\Mat{X},\Mat{X})$ is not class-dependent,
the following expression can be obtained:

\begin{equation}
    \varepsilon_i(\Mat{X})=
     -2\Vec{y}_i^T\Vec{\mathcal{K}(X,\mathbb{D})}+\Vec{y}_i^T \Mat{\mathbb{K}(\mathbb{D},\mathbb{D})} \Vec{y}_i .
     \label{eqn:sparse_classification2}
\end{equation}%

Alternatively, the similarity between query sample $\Mat{X}$ to class $i$
can be defined as \mbox{$S_i(\Mat{X})= h(\Vec{y}_i)$}.
The function {$h(\cdot)$} could be a linear function like \mbox{$h(\Vec{y}_i) = \Vec{y}_i^T\Vec{1}_{N\times1}$}
or even a non-linear one like {$h(\Vec{y}_i) = \max\left(\Vec{y}_i\right)$}.
Preliminary experiments suggest that Eqn.~\eqref{eqn:sparse_classification2} leads to higher classification accuracies
when compared to the aforementioned alternatives.

\subsection{Computational Complexity}

In terms of computational complexity, we note that the complexity of computing the determinant of an $n \times n$ matrix through 
Cholesky decomposition is $O(\frac{1}{3}n^3)$. Therefore, computing $S(\Mat{X},\Mat{D}_i)$ by storing the determinant of dictionary atoms 
during learning costs $O(\frac{2}{3}n^3)$.

For the $J$ divergence, we note that the inverse of an $n \times n$ SPD matrix can be computed through Cholesky decomposition
with $\frac{1}{2}n^3$ flops. Therefore, $J(\Mat{X},\Mat{D}_i)$ can be computed in $2n^{2.3} + \frac{1}{2}n^3$ flops if matrix multiplication is done efficiently.
As a result, computing the $J$ divergence is cheaper than computing $S$ divergence for SPD matrices of size less than 35.

The complexity of sparse coding is dictated by $\Vec{\mathcal{K}(X,\mathbb{D})}$ in Eqn.~\eqref{eqn:sparse_classification2}. Neglecting the 
complexity of the exponential in kernel functions, the complexity of generating Eqn.~\eqref{eqn:sparse_classification2} is  
$O\big(N(2n^{2.3} + \frac{1}{2}n^3)\big)$ for $J$ divergence and $O\big(\frac{2N}{3}n^3\big)$ for $S$ divergence.

Note that while the computational complexity is cubic in $n$, it is linear in $N$, \ie, number of dictionary atoms.
To give the reader an idea on the speed of the proposed methods,
it is worth mentioning that performing sparse coding on $93 \times 93$ covariance descriptors used in~\textsection~\ref{sec:exp_action_rec}
took less than 10 and 7 seconds with Jeffrey and Stein divergences, respectively (on an Intel i7 machine using Matlab).
Performing a simple nearest neighbor search using AIRM required more than 75 seconds on the same dataset.

\section{Dictionary Learning}
\label{sec:dic_learning}

Given a finite set of observations 
$\mathbb{X} = \{ \Mat{X}_i\}_{i=1}^m,\; \Mat{X}_i \in \SPD{n}$,
learning a dictionary	$\mathbb{D} = \{ \Mat{D}_i\}_{i=1}^N,\; \Mat{D}_i \in \SPD{n}$
by embedding SPD manifolds into Hilbert space can be formulated as minimizing the following energy function
with respect to $\mathbb{D}$:

\noindent
\begin{equation}
	f(\mathbb{X},\mathbb{D}) \triangleq  \sum\nolimits_{i=1}^{m} l_\phi(\Mat{X}_i,\mathbb{D}).
    \label{eqn:SPD_dic_learning}
\end{equation}

Here, $l_\phi(\Mat{X},\mathbb{D})$ is the loss function defined in Eqn.~\eqref{eqn:kernel_sparse_coding}.
$f(\mathbb{X},\mathbb{D})$ should be small if $\mathbb{D}$ is ``good''
at representing the signals $\Mat{X}_i$. 
Among the various solutions to the problem of dictionary learning in Euclidean spaces, iterative methods
like K-SVD have received much attention~\cite{ELAD_SR_BOOK_2010}.
Borrowing the idea from Euclidean spaces,
we propose to minimize the energy in Eqn.~\eqref{eqn:SPD_dic_learning} iteratively.

To this end, we first initialize the dictionary $\mathbb{D}$ randomly.
It is also possible to use intrinsic {\it k}-means clustering using the Karcher mean~\cite{Pennec_IJCV_2006} to initialize the dictionary.
Each iteration of dictionary learning then constitutes of two parts,
namely a sparse coding step and a dictionary update step.
In the sparse coding step,	the dictionary $\mathbb{D}$ is fixed and sparse codes, $\{\Vec{y}_i\}_{i=1}^m$ are computed as 
discussed in~\textsection~\ref{sec:sparse_coding}.
In the dictionary update step,	$\{\Vec{y}_i\}_{i=1}^m$ are held fixed while $\mathbb{D}$ is updated,
with each dictionary atom updated independently.
This resembles the Expectation Maximization (EM) algorithm~\cite{Dempster_1977} in nature. 
In the following subsections, we discuss how dictionary atoms can be updated for both $J$ and $S$ divergences.

\subsection{Dictionary Updates for $J$~Divergence}
\label{sec:sub_dic_learning_j_div}

As mentioned above, to update $\Mat{D}_r$,
we keep $\Mat{D}_j, \;j \neq r$
and the sparse codes $\{\Vec{y}_i\}_{i=1}^m$  in Eqn.~\eqref{eqn:SPD_dic_learning} fixed.
Generally speaking, one can update $\Mat{D}_r$ using gradient descend algorithms on SPD manifolds.
This can be done at iteration~$t$ by exploiting the tangent space at $\Mat{D}_r^{(t)}$ 
and moving along the direction of steepest descent and utilizing the exponential map to obtain $\Mat{D}_r^{(t+1)}$
as a point on $\SPD{n}$. 

In this paper, we propose to learn the dictionary in an online manner.
Our proposal results in an analytical and closed-form solution for updating dictionary atoms one by one. 
In contrast to \cite{Vemuri_ICML_2013}, our formulation does not exploit the tangent bundle and exponential maps,
and is hence faster and more scalable.
By fixing $\Mat{D}_j, \;j \neq r$ and $\{\Vec{y}_i\}_{i=1}^m$,
the derivative of Eqn.~\eqref{eqn:SPD_dic_learning} with respect to  $\Mat{D}_r$  can be computed as

\noindent
\begin{align}
	\frac{\partial f(\mathbb{X},\mathbb{D})}{\partial \Mat{D}_r} &= \sum\limits_{i=1}^{m} 
	\frac{\partial l_\phi(\Mat{X}_i,\mathbb{D})}{\partial \Mat{D}_r} \label{eqn:deriv_dic_eqn} \\
	&= \sum\limits_{i=1}^{m} \Vec{y}_{i,r} \Bigg(\sum\limits_{j=1}^{N} \Vec{y}_{i,j}\frac{\partial k(\Mat{D}_j,\Mat{D}_r)}{\partial \Mat{D}_r}
	-2\frac{\partial k(\Mat{X}_i,\Mat{D}_r)}{\partial \Mat{D}_r}\Bigg)\notag .
\end{align}%

\noindent
For the $J$~divergence, we note that

\noindent
\begin{equation}
	\nabla_{\Mat{X}} J(\Mat{X},\Mat{Y}) = \frac{1}{2}(\Mat{Y}^{-1} - \Mat{X}^{-1}\Mat{Y}\Mat{X}^{-1}).
	\label{eqn:gradient_J}
\end{equation}

\noindent
Therefore, 
\begin{equation}
	\frac{\partial k_J(\Mat{X},\Mat{Y})}{\partial \Mat{X}}  = -\frac{1}{2}\beta k_J(\Mat{X},\Mat{Y}) 
	(\Mat{Y}^{-1} - \Mat{X}^{-1}\Mat{Y}\Mat{X}^{-1}).
	\label{eqn:gradient_J_2}
\end{equation}

\noindent
Plugging Eqn.~\eqref{eqn:gradient_J_2} into Eqn.~\eqref{eqn:deriv_dic_eqn} and defining

\noindent
\begin{align}
	\Mat{P} \hspace{-0.5ex}&= \hspace{-1ex}\sum\limits_{i=1}^{m} \Vec{y}_{i,r} \bigg(
	\sum\limits_{j=1}^{N} \Vec{y}_{i,j}k_J(\Mat{D}_j,\Mat{D}_r)\Mat{D}_j^{-1}
	-2k_J(\Mat{X}_i,\Mat{D}_r)\Mat{X}_i^{-1} 
	\hspace{-0.5ex}  \bigg), \notag\\
	\Mat{Q} \hspace{-0.5ex}&= \hspace{-0.8ex}\sum\limits_{i=1}^{m}\hspace{-0.3ex} \Vec{y}_{i,r} \hspace{-0.5ex}
	\bigg(\hspace{-0.3ex}\sum\limits_{j=1}^{N} \hspace{-0.5ex}	\Vec{y}_{i,j}k_J(\Mat{D}_j,\Mat{D}_r\hspace{-0.4ex})\Mat{D}_j
	\hspace{-0.5ex}- \hspace{-0.5ex}2k_J(\Mat{X}_i,\Mat{D}_r\hspace{-0.4ex})\Mat{X}_i
	\hspace{-0.5ex}\bigg)\hspace{-0.5ex},
	\label{eqn:dic_learning_j_comp}
\end{align}%

\noindent
then the root of Eqn.~\eqref{eqn:deriv_dic_eqn}, \ie, $\partial f(\mathbb{X},\mathbb{D})/\partial \Mat{D}_r = 0$ can be written as: 
\begin{equation}
	\Mat{D}_r^{-1}\Mat{Q}\Mat{D}_r^{-1} = \Mat{P}.
	\label{eqn:sol_dic_learning_j_eq}
\end{equation}
This equation is identified as a \emph{Riccati} equation~\cite{BHATIA_2007}.
Its solution is positive definite and given as

\noindent
\begin{equation}
	\Mat{D}_r = \Mat{Q}^{1/2}\big(\Mat{Q}^{-1/2} \Mat{P}^{-1} \Mat{Q}^{-1/2}\big)^{1/2}\Mat{Q}^{1/2},
	\label{eqn:sol_dic_learning_j_eq2}
\end{equation}

\noindent
provided that both $\Mat{P}$ and $\Mat{Q}$ are positive definite. 
We note that in deriving the solution, we have assumed that $k_J(\Mat{D}_r,\cdot)$ at iteration $t$ 
can be replaced by $k_J(\Mat{D}_r^{t-1},\cdot)$ and hence $k_J(\Mat{D}_r,\cdot)$ are treated as scalars.

\subsection{Dictionary Updates for $S$~Divergence}
\label{sec:sub_dic_learning_s_div}
Similar to~\textsection~\ref{sec:sub_dic_learning_j_div}, we need to compute the gradient of Eqn.~\eqref{eqn:SPD_dic_learning}
with respect to $\Mat{D}_r$, while $\{\Vec{y}_i\}_{i=1}^m$ and other atoms are fixed. Noting that
\begin{equation}
	\nabla_{\Mat{X}} S(\Mat{X},\Mat{Y}) = (\Mat{X}+\Mat{Y})^{-1} - \frac{1}{2} \Mat{X}^{-1},
	\label{eqn:gradient_Stein}
\end{equation}
\noindent
the solution of $\partial f(\mathbb{X},\mathbb{D})/\partial \Mat{D}_r = 0$ with $k_S(\cdot,\cdot)$ can be written as:

\begin{align}
       	&\sum \limits_{i=1}^{m}\Vec{y}_{i,r}\Bigg(
       	2k_S(\Mat{X}_i,\Mat{D}_r)\Big( (\Mat{X}_i+\Mat{D}_r)^{-1} -\frac{1}{2} \Mat{D}_r^{-1} \Big)\Bigg) =\notag \\
    	&\sum \limits_{i=1}^{m}\Vec{y}_{i,r}\Bigg(\sum \limits_{j=1}^{N} \Vec{y}_{i,j} k_S(\Mat{D}_j,\Mat{D}_r ) 
    	\Big( (\Mat{D}_j+\Mat{D}_r)^{-1} - \frac{1}{2} \Mat{D}_r^{-1}\Big)
    	\Bigg). 
    \label{eqn:dic_stein_lrn3}
\end{align}%

\noindent
Since Eqn.~\eqref{eqn:dic_stein_lrn3} contains inverses and kernel values, a~closed-form solution for computing $\Mat{D}_r$ cannot be sought.
As such, we propose an alternative solution	by exploiting previous values of $\left ( \Mat{D}_i + \Mat{D}_r \right )^{-1}$	
in the update step.
More specifically, rearranging Eqn.~\eqref{eqn:dic_stein_lrn3} and replacing	$k(\cdot,\Mat{D}_r)$
as well as	$ \left ( \Mat{D}_i + \Mat{D}_r \right )^{-1}$	by their previous values,
atom  $\Mat{D}_r$ at iteration $t+1$ is updated according to:

\begin{equation}
   	\Mat{D}_r^{(t+1)} =
   	\frac{2\Mat{P}^{-1}}{\sum\limits_{i=1}^{m}\Vec{y}_{i,r}\Big(
   	2k_S(\Mat{X}_i,\Mat{D}_r)-\sum\limits_{j=1}^{N}\Vec{y}_{i,j} k_S(\Mat{D}_j,\Mat{D}_r)\Big)}
   	,    	
   	\label{eqn:dic_lrn4}
\end{equation}%

\noindent
where,
\begin{align}
    \Mat{P} = \sum \limits_{i=1}^{m}\Vec{y}_{i,r} \Bigg(
    2k_S(\Mat{X}_i,\Mat{D}_r)\Big( \Mat{X}_i + \Mat{D}_r^{(t)} \Big)^{-1} \notag\\
    - \sum\limits_{j=1}^{N}\Vec{y}_{i,j}k_S(\Mat{D}_j,\Mat{D}_r)\Big( \Mat{D}_j + \Mat{D}_r^{(t)} \Big)^{-1}\Bigg).
    \label{eqn:dic_lrn5}
\end{align}

\subsection{Practical Considerations}
\label{sec:sub_prac_cons}

The dictionary update in Eqn.~\eqref{eqn:sol_dic_learning_j_eq2} results in an SPD matrix provided that matrices $\Mat{P}$ and $\Mat{Q}$
are SPD. In practice, this might not be the case and as such projection to the positive definite cone is required.
The same argument holds for Eqn.~\eqref{eqn:dic_lrn4}. 
Given an arbitrary square matrix $\Mat{A} \in \mathbb{R}^{n \times n}$, the problem of finding the closest SPD matrix 
to $\Mat{A}$ has received considerable attention in the literature (\cf,~\cite{Higham_1988}). 
While projecting onto positive definite cone can be achieved by thresholding (\ie, replacing negative eigenvalues by a small positive number),
a more principal approach can be used as follows.
If square matrix $\Mat{X}$ is positive definite then $\Mat{X} + \Mat{X}^T$ is also positive definite.
As such, the following convex problem can be solved to obtain the closest SPD matrix $\Mat{X}$ to the square matrix $\Mat{A}$ by using a solver like CVX~\cite{CVX}. 
\begin{align}
&\min \|\Mat{A} - \Mat{X}\|_F \notag\\
&\mathrm{s.t.}~~\Mat{X} + \Mat{X}^T \succ 0
\end{align}
We note that the formulation provided here works for non-symmetric matrix $\Mat{A}$ as well.
This is again useful in practice as numerical issues might create non-symmetric matrices (\eg, $\Mat{P}$ and $\Mat{Q}$ in Eqn.~\eqref{eqn:sol_dic_learning_j_eq2} might not become symmetric
due to the limited numerical accuracy in a computational implementation).

\section{Experiments}
\label{sec:experiments}

Two sets of experiments\footnote{The corresponding Matlab/Octave source code is available at \mbox{\url{http://nicta.com.au/people/mharandi}}}
are presented in this section.
In the first set, we evaluate the performance of the proposed sparse coding methods (as described in~\textsection~\ref{sec:sparse_coding}) without dictionary learning. 
This is to contrast sparse coding to previous state-of-the-art methods on several popular closed-set classification tasks.
To this end, each point in the training set is considered as an atom in the dictionary.
Since the atoms in the dictionary are labeled in this case, the residual error approach for classification 
(as described in~\textsection~\ref{sec:sparse_classification}) will be used to determine the label of a query point.
In the second set of experiments, the performance of the sparse coding methods is evaluated
in conjunction with the proposed dictionary learning algorithms described in~\textsection~\ref{sec:dic_learning}.
For brevity, we denote Riemannian sparse representation with $J$~divergence as \mbox{RSR-J},
and the $S$~divergence counterpart as \mbox{RSR-S}. 

The first priority of the experiments is to contrast the proposed methods against recent techniques designed to work on SPD manifolds.
That is, the tasks and consequently the datasets were chosen to enable fair comparisons against state-of-the-art SPD methods.
While exploring other visual tasks such as face verification~\cite{Chen_CVPR_2013_Blessing}
is beyond the scope of this paper, it is an interesting path to pursue in future work. 
\subsection{Sparse Coding}%
Below, we compare and contrast the performance of \mbox{RSR-J} and \mbox{RSR-S} methods against state-of-the-art techniques in five classification tasks,
namely action recognition from 3D skeleton data, face recognition,  material classification, person re-identification and texture categorization.

\subsubsection{Action Recognition from 3D Skeleton Sequences}
\label{sec:exp_action_rec}

We used the motion capture HDM05 dataset~\cite{HDM05_Doc} for the task of action recognition from skeleton data.
Each action is encoded by the locations of 31 joints over time, with the speed of 120 frames per second.
Given an action by $K$ joints over $m$ frames, we extracted the joint covariance descriptor~\cite{Husse_IJCAI_2013} which is an SPD matrix
of size $3K \times 3K$ as follows.
Let $x_i(t)$, $y_i(t)$ and $z_i(t)$ be the $x$, $y$, and $z$ coordinates of the $i$-{th} joint at frame
$t$. Let $\Vec{f}(t)$ be the vector of all joint locations at time $t$, \ie, 
{\small $\Vec{f}(t)~=~\left( x_1(t), \cdots, x_K(t),y_1(t),\cdots, y_K(t),z_1(t),  \cdots, z_K(t) \right)^T$}, which has $3K$ elements.
The action represented over $m$ frames is then described by the covariance of vectors $\Vec{f}(t)$.

We used 3 subjects (140 action instances) for training, 
and the remaining 2 subjects (109 action instances) for testing. 
The set of actions used in this experiment is: `clap above head', `deposit floor', `elbow to knee', `grab high', 
`hop both legs', `jog', `kick forward', `lie down floor', `rotate both arms backward', `sit down chair',
`sneak', `squat', `stand up lie' and `throw basketball'. 

In Table~\ref{tab:table_HDM05_performance} we compare the performance of \mbox{RSR-J} and \mbox{RSR-S} against logEuc-SC~\cite{Guo_TIP13} and Cov3DJ~\cite{Husse_IJCAI_2013}.
The TSC algorithm~\cite{TSC_PAMI_2014} does not scale well to large SPD matrices and thus is not considered here.
Cov3DJ encodes the relationship between joint movement and time by deploying multiple covariance matrices over sub-sequences in a hierarchical fashion.
The results show that in this case \mbox{RSR-J} is better than \mbox{RSR-S}.
Furthermore, both \mbox{RSR-J} and \mbox{RSR-S} outperform logEuc-SC and Cov3DJ.
 
\begin{table}[!tb]
  	\centering
 	\caption{\small Recognition accuracy (in \%) for the HDM05-MOCAP dataset~\cite{HDM05_Doc}}
    \begin{tabular}{lc}
    	\toprule
    	{\bf Method} &{\bf Recognition Accuracy }\\
    	\toprule  		
		{\bf logEuc-SC~\cite{Guo_TIP13}}    &$89.9\%$\\
		{\bf Cov3DJ~\cite{Husse_IJCAI_2013}}    &$95.4\%$\\
    	\midrule
    	{\bf \mbox{RSR-J}}           &$\bf 98.2\%$\\    			
    	{\bf \mbox{RSR-S}}           &$97.3\%$\\
		\bottomrule	
    \end{tabular}

    \label{tab:table_HDM05_performance}
\end{table}

\begin{figure}[!tb]
	\centering
	\includegraphics[scale = 0.85]{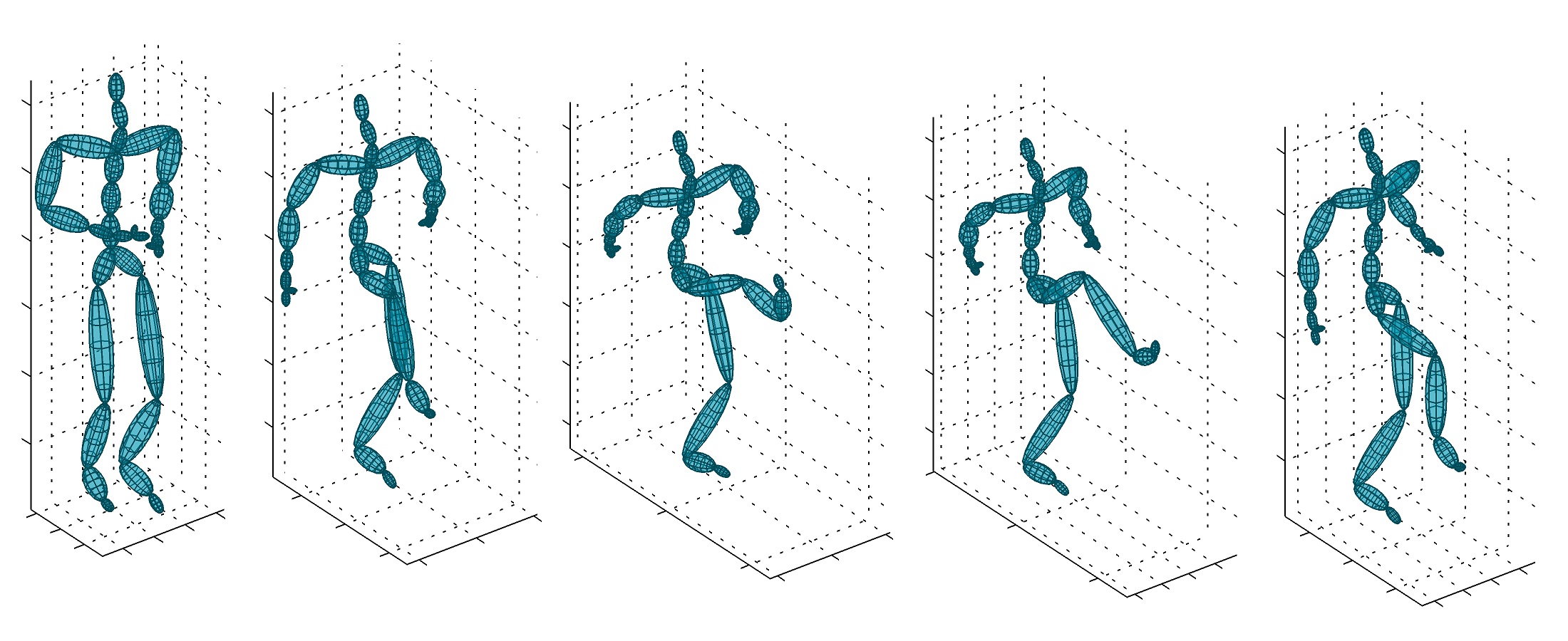}
	\caption{\small Example of a kicking action from the HDM05 action dataset~\cite{HDM05_Doc}.}
	\label{fig:MOCAP_Dataset}
\end{figure}		

\subsubsection{Face Recognition}
\label{sec:exp_face_rec}
We used the `b' subset of the FERET dataset~\cite{FERET_Dataset}, which includes 1800 images from 200 subjects.
The images were closely cropped around the face and downsampled to  $64 \times 64$.
Examples are shown in Figure~\ref{fig:FERET_Dataset}.

We performed four tests with various pose angles.
Training data was composed of images marked `ba', `bj' and `bk' (\ie, frontal faces with expression and illumination variations).
Images with `bd', `be', `bf', and `bg' labels (\ie, \mbox{non-frontal} faces) were used as test data.

\def \FERET_SCALE {0.125}
\begin{figure}[!tb]	
	\centering
	\begin{subfigure}{\FERET_SCALE \columnwidth}
		\includegraphics[width = \columnwidth]{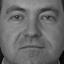}
		\caption{ba}
	\end{subfigure}
	\begin{subfigure}{\FERET_SCALE \columnwidth}
		\includegraphics[width = \columnwidth]{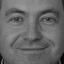}
		\caption{bj}
	\end{subfigure}
	\begin{subfigure}{\FERET_SCALE \columnwidth}
		\includegraphics[width = \columnwidth]{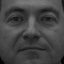}
		\caption{bk}
	\end{subfigure}
	\begin{subfigure}{\FERET_SCALE \columnwidth}
		\includegraphics[width = \columnwidth]{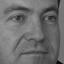}
		\caption{bd}
	\end{subfigure}
	\begin{subfigure}{\FERET_SCALE \columnwidth}
		\includegraphics[width = \columnwidth]{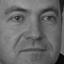}
		\caption{be}
	\end{subfigure}
	\begin{subfigure}{\FERET_SCALE \columnwidth}
		\includegraphics[width = \columnwidth]{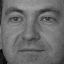}
		\caption{bf}
	\end{subfigure}	
	\begin{subfigure}{\FERET_SCALE \columnwidth}
		\includegraphics[width = \columnwidth]{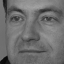}
		\caption{bg}
	\end{subfigure}		
	\caption{Examples from the FERET face dataset~\cite{FERET_Dataset}}
	\label{fig:FERET_Dataset}
\end{figure}

Each face image is described by a $43 \times 43$ SPD matrix
using the following features:
\noindent
\begin{equation*}
	\Vec{f}_{x,y} =	\left(~ I(x,y),~ x,~ y,~ |G_{0,0}(x,y)|,~\cdots,~|G_{4,7}(x,y)| ~\right)^T\;,
\end{equation*}%
\noindent
where {\small $I(x,y)$} is the intensity value at position $(x,y)$,
$|\cdot|$ denotes the magnitude of a  complex value
and $G_{u,v}{(x,y)}$ is the response of a 2D Gabor wavelet
centered at $(x,y)$ with orientation $u$ and scale $v$.
In this work, we followed~\cite{Pang_TCSVT_2008} and generated 40 Gabor filters in 8 orientations and 5 scales. 
		
The proposed methods are compared against TSC~\cite{TSC_PAMI_2014},
logEuc-SC~\cite{Guo_TIP13},
Sparse Representation-based Classification (SRC)~\cite{Wright_2009_PAMI}
and its Gabor-based extension (GSRC)~\cite{Yang_ECCV_2010_GSRC}.
For SRC, PCA was used to reduce the dimensionality of data.
We evaluated the performance of SRC for various dimensions of PCA space and the maximum performance is reported.
For the GSRC algorithm~\cite{Yang_ECCV_2010_GSRC}, we followed the recommendations of the authors for the downsampling factor in Gabor filtering. 
As for the logEuc-SC, we consider a kernel extension of the original algorithm.
In other words, instead of directly using $\log(\cdot)$ representations in a sparse coding framework as done in~\cite{Guo_TIP13},
we consider a kernel extension on $\log$ representations using an RBF kernel.
The kernel extension of sparse coding is discussed in depth in~\cite{Gao_ECCV_2010,Nguyen_TIP_2013}. 
This enhances the results in all cases and makes the logEuc-SC and RSR methods more comparable.

Table~\ref{tab:table_FERET_performance} shows the performance of all the studied methods for the task of face recognition. 
Both \mbox{RSR-J} and \mbox{RSR-S} outperform other methods,
with \mbox{RSR-S} being marginally better than \mbox{RSR-J}.

\begin{table}[!tb]
 	\caption    {    
		Recognition accuracy (in \%) for the FERET face dataset~\cite{FERET_Dataset}.
    }
  	\centering
    \begin{tabular}{lccccc}
    	\toprule
    	{\bf Method} 							&{\bf bd }	&{\bf be }	&{\bf bf }	&{\bf bg }	 &{\bf average}\\
    	\toprule  		
    	{\bf SRC~\cite{Wright_2009_PAMI}}      	&$27.5\%$	&$55.5\%$	&$61.0\%$	&$26.0\%$	 &$42.5\%$\\  
    	{\bf GSRC~\cite{Yang_ECCV_2010_GSRC}}   &$77.0\%$	&$93.5\%$	&$97.0\%$	&$79.0\%$	 &$86.6\%$\\
    	\midrule
    	{\bf logEuc-SC~\cite{Guo_TIP13}}   		&$74.0\%$	&$94.0\%$	&$97.5\%$	&$80.5\%$	 &$86.5\%$\\
    	{\bf TSC~\cite{TSC_PAMI_2014}}   			&$36.0\%$	&$73.0\%$	&$73.5\%$	&$44.5\%$	 &$56.8\%$\\
    	\midrule
		{\bf RSR-J}      					&$\bf 82.5\%$	&$94.5\%$	&$\bf 98.0\%$	&$83.5\%$	 &$89.6\%$\\    			
		{\bf RSR-S}      					&$79.5\%$	&$\bf96.5\%$	&$97.5\%$	&$\bf 86.0\%$	 &$\bf 89.9\%$\\ 
		\bottomrule	
	\end{tabular}
	\label{tab:table_FERET_performance}
\end{table}

\subsubsection{Material Categorization}
\label{sec:exp_material_cat}
We used the Flickr dataset~\cite{Flickr_Dataset} for the task of material categorization.
The dataset contains ten categories of materials: 
\textit{fabric}, \textit{foliage}, \textit{glass}, \textit{leather},
\textit{metal}, \textit{paper}, \textit{plastic}, \textit{stone}, \textit{water} and \textit{wood}. 
Each category has 100 images, 50 of which are close-up views and the remaining 50 are views at object-scale 
(see Figure~\ref{fig:FM_Dataset} for examples).
A binary, human-labeled mask is provided for each image in the dataset, 
describing the location of the object in the image. We only consider pixels inside this binary mask for material recognition
and disregard all background pixels. 
SIFT~\cite{SIFT_IJCV_2004} features have recently been shown to be robust and discriminative for material classification~\cite{UIUC_Dataset}.
We therefore constructed RCMs of size $155\times155$ using 128 dimensional SIFT features (only from gray-scaled images) and 27 dimensional color descriptors.
To this end, SIFT descriptors were computed at points on a regular grid with 5 pixel spacing.
The color descriptor was obtained by simply stacking colors from $3\times3$ patches centered at grid points. 
	
Table~\ref{tab:table_FMD_performance} compares the performance of the proposed methods against
the state-of-the-art non-parametric geometric detail extraction method (SD)~\cite{UIUC_Dataset},
augmented Latent Dirichlet Allocation (aLDA)~\cite{FMD_CVPR_2010}, 
and the texton-based representation introduced in~\cite{Varma_2009_PAMI}. 
The results indicate that RSR-S considerably outperforms previous state-of-the-art approaches.
We also note that RSR-J outperforms methods proposed in~\cite{Varma_2009_PAMI,UIUC_Dataset}
by a large margin and is only slightly worse than the aLDA algorithm~\cite{FMD_CVPR_2010}.

\def \FMD_SCALE {0.18}
\begin{figure}[!tb]
	\begin{minipage}{\columnwidth}\centering
		\includegraphics[width = \FMD_SCALE \columnwidth]{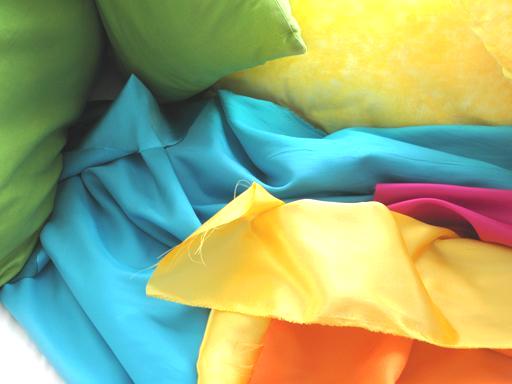}
		\includegraphics[width = \FMD_SCALE \columnwidth]{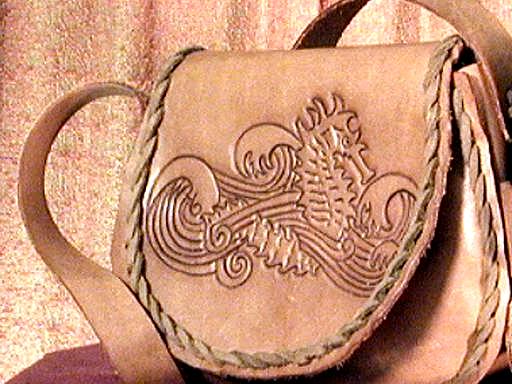}
		\includegraphics[width = \FMD_SCALE \columnwidth]{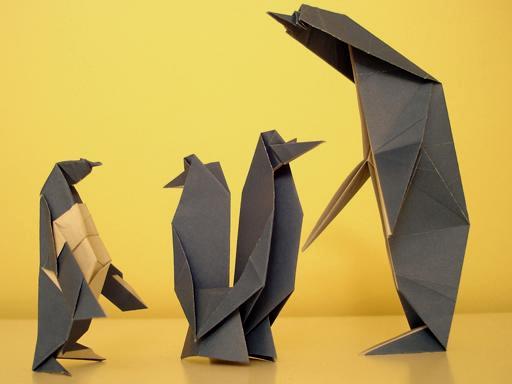}
		\includegraphics[width = \FMD_SCALE \columnwidth]{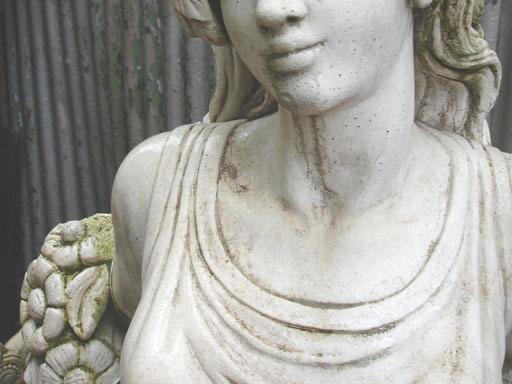}
		\includegraphics[width = \FMD_SCALE \columnwidth]{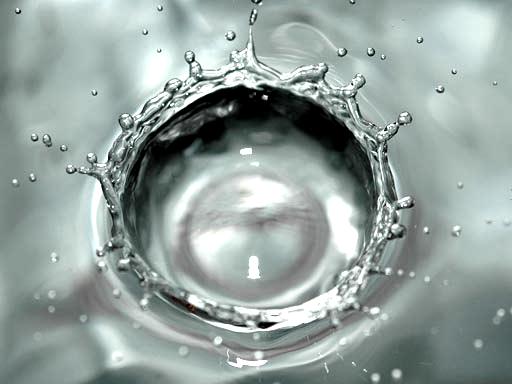}
	\end{minipage}
	\caption{Examples from the Flickr dataset~\cite{Flickr_Dataset}.}
	\label{fig:FM_Dataset}
\end{figure}

\begin{table}[!tb]
	\caption{    
		\small
	    Recognition accuracy (in \%) along its standard deviation for the Flickr dataset~\cite{Flickr_Dataset}.
    }
  	\centering
    \begin{tabular}{lc}
    	\toprule
    	{\bf Method} &{\bf Recognition Acc.}\\
    	\toprule    		
    	{\bf VZ~\cite{Varma_2009_PAMI}}                          &$23.8\% \pm N/A$\\
	   	{\bf VZ-augmented~\cite{FMD_CVPR_2010}}                  &$37.4\% \pm N/A$\\
	   	{\bf SD~\cite{UIUC_Dataset}}                          	 &$29.9\% \pm N/A$\\
    	{\bf aLDA~\cite{FMD_CVPR_2010}	}                  		 &$44.6\% \pm N/A$\\ 
   		{\bf RSR-J}          		 	 						&$44.0\% \pm 3.0$\\
   		{\bf RSR-S}          		 	 						&$\bf 51.4\% \pm 1.9$\\
	    \bottomrule	
   	\end{tabular}
   	\label{tab:table_FMD_performance}
\end{table}

\subsubsection{Person Re-identification}
\label{sec:exp_person_rec}
We used the modified ETHZ dataset~\cite{Schwartz_ETHZ}.
The original ETHZ dataset was captured using a moving camera~\cite{ETHZ_ICCV},
providing a range of variations in the appearance of people.
The dataset is structured into 3 sequences.
Sequence~1 contains 83 pedestrians (4,857 images),
Sequence~2 contains 35 pedestrians (1,936 images),
and
Sequence~3 contains 28 pedestrians (1,762 images).
See left panel of Fig.~\ref{fig:ETHZ_results} for examples.

We downsampled all images to \mbox{\small $64 \times 32$} pixels.
For each subject we randomly selected 10 images for training and used the rest for testing.
Random selection of training and testing data was repeated 20 times to obtain reliable statistics.
To describe each image, the covariance descriptor was computed using the following features:

\noindent
\begin{equation*}
\Vec{f}_{\Vec{u}}
\mbox{=}
\big (
  \Vec{u},~
  R_{\Vec{u}},~   G_{\Vec{u}},~   B_{\Vec{u}},~
  \dot R_{\Vec{u}},~  \dot G_{\Vec{u}},~  \dot B_{\Vec{u}},~
  \ddot R_{\Vec{u}},~ \ddot G_{\Vec{u}},~  \ddot B_{\Vec{u}}~
\big)^T
\end{equation*}%

\noindent
where {$\Vec{u} = (x,y)^T$} is the position of a pixel,
while
{ $R_{\Vec{u}}$}, {$G_{\Vec{u}}$} and {$B_{\Vec{u}}$}
represent the corresponding color information.
The gradient and Laplacian for color {\footnotesize $C$}
are represented by
$\dot C_{\Vec{u}} \mbox{=} \big( \left|{\partial C} \middle/ {\partial x}\right|, \left|{\partial C} \middle/ {\partial y}\right| \big)$
and
$\ddot C_{\Vec{u}} \mbox{=} \big( \left|{\partial^2 C} \middle/ {\partial x^2}\right|, \left|{\partial^2 C} \middle/ {\partial y^2}\right| \big)$,
respectively.

We compared the proposed RSR methods with
several techniques previously used for pedestrian detection:
Symmetry-Driven Accumulation of Local Features (SDALF)~\cite{Bazzani_CVIU13},
Riemannian Locality Preserving Projection (RLPP)~\cite{Harandi_WACV_2012},
and log-Euclidean sparse coding~\cite{Guo_TIP13}.
The results for TSC~\cite{TSC_PAMI_2014} could not be generated in a timely manner
due to the heavy computational load of the algorithm.

Results for the first two sequences are shown in Fig.~\ref{fig:ETHZ_results},
in terms of cumulative matching characteristic (CMC) curves.
The CMC curve represents the expectation of finding the correct match in the top {\small $n$} matches.
The proposed \mbox{RSR-S} method obtains the highest accuracy on both sequences. 
\mbox{RSR-J} outperforms SDALF, RLPP and log-Euclidean sparse coding on sequence one.
For the second sequence, RLPP and SDALF perform better than \mbox{RSR-J} for low ranks while
\mbox{RSR-J} outperforms them for rank higher than two.

For Sequence 3 (not shown), very similar performances are obtained by SDALF, RLPP and the proposed methods.
For this sequence, \mbox{RSR-J} and \mbox{RSR-S} achieve rank 1 accuracy of $98.3\%$ and $98.7\%$, respectively.
The CMC curves are almost saturated at perfect recognition at rank 3 for both \mbox{RSR-J} and \mbox{RSR-S} methods.

\def \ETHZ_SCALE {0.375}
\begin{figure*}[!tb]\centering
	\begin{subfigure}{0.20 \textwidth}
		\includegraphics[width= 1.0 \columnwidth,keepaspectratio]{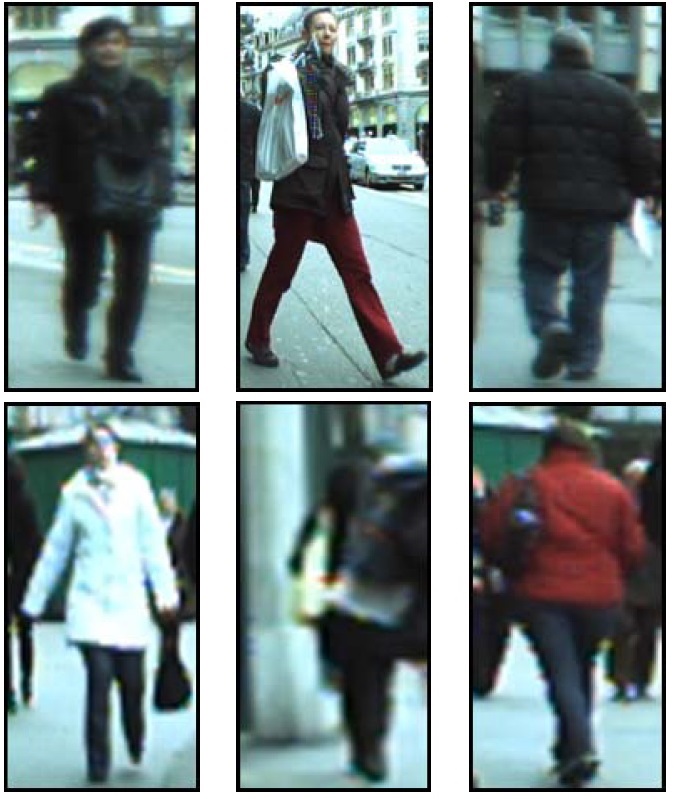}
		\caption{Sample images}
	\end{subfigure}	
	~
	\begin{subfigure}{\ETHZ_SCALE \textwidth}
		\includegraphics[width = \columnwidth,keepaspectratio]{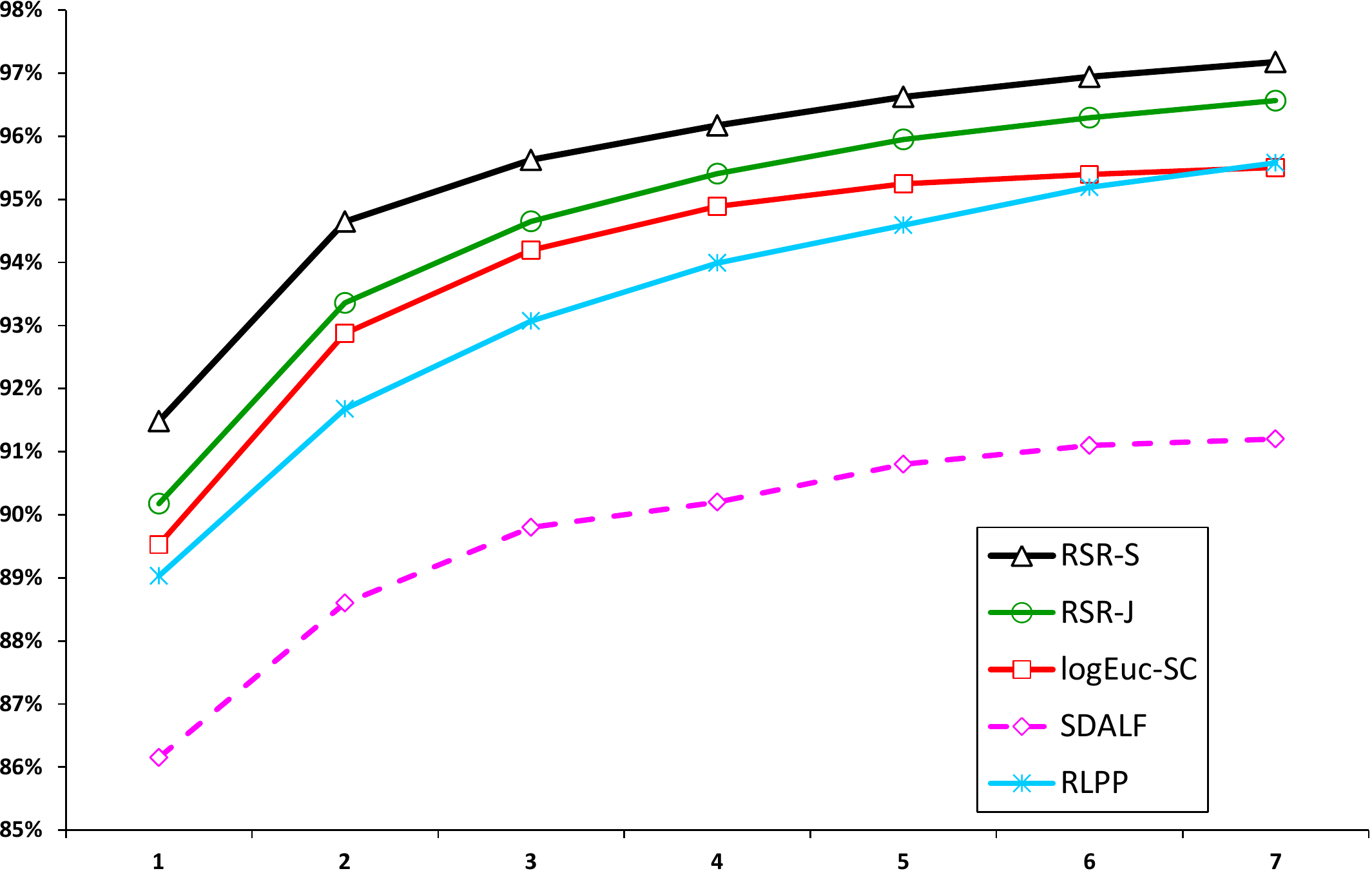}
		\caption{Seq.~\#1}
	\end{subfigure}
	~
	\begin{subfigure}{\ETHZ_SCALE \textwidth}
		\includegraphics[width = \columnwidth,keepaspectratio]{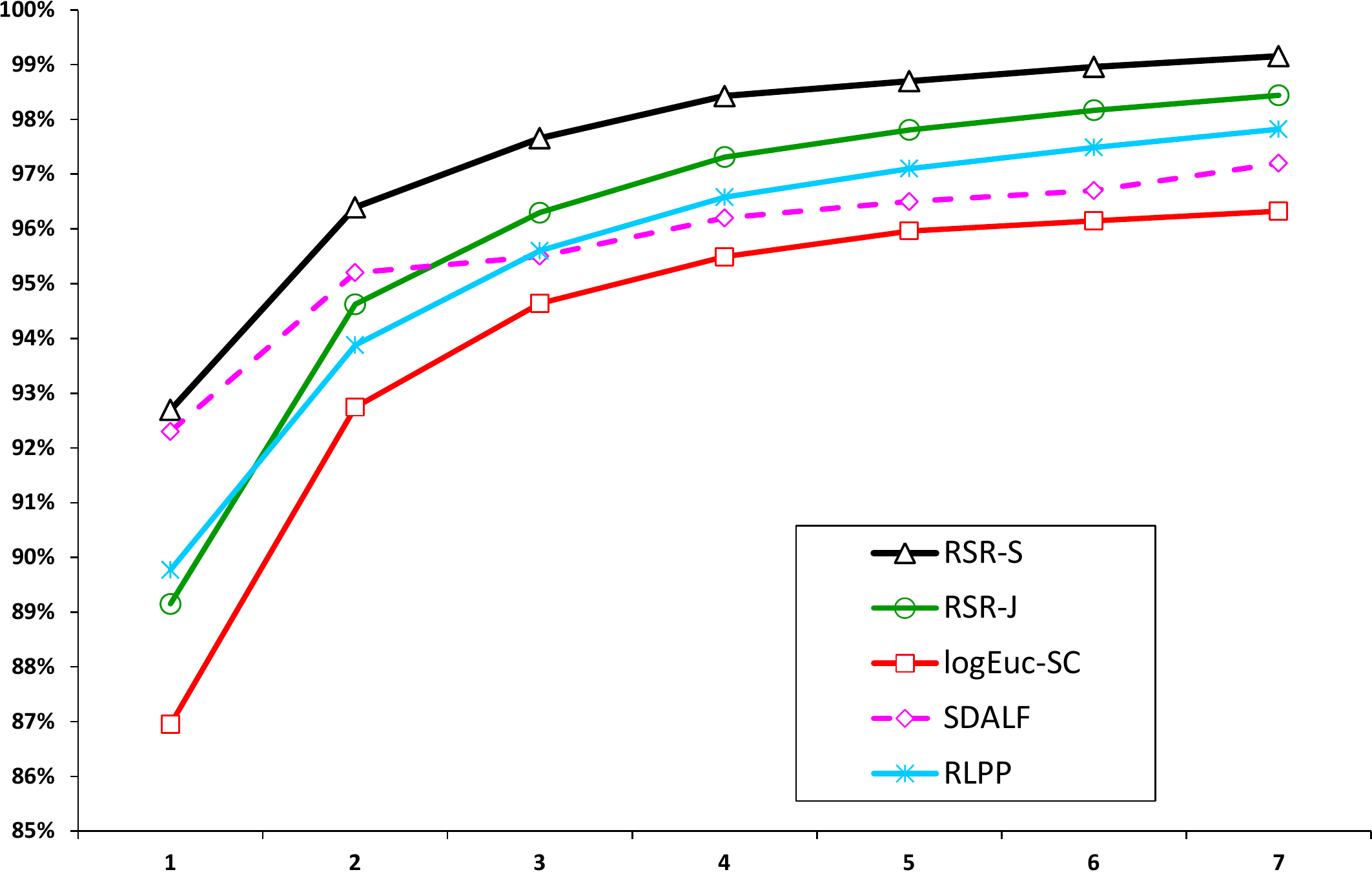}
		\caption{Seq.~\#2}
	\end{subfigure}
	\caption
        {\small
        Person Re-identification using the ETHZ dataset~\cite{ETHZ_ICCV}. Left column, examples of pedestrians in the ETHZ dataset.
        Middle column, results on Seq.~\#1; right column, on Seq.~\#2. 
        The proposed \mbox{RSR-J} and \mbox{RSR-S} methods are compared with SDALF~\cite{Bazzani_CVIU13}, RLPP~\cite{Harandi_WACV_2012} and log-Euclidean sparse coding (logEuc-SC)~\cite{Guo_TIP13}.
        }
    \label{fig:ETHZ_results}
\end{figure*}

\subsubsection{Texture Classification}
\label{sec:exp_texture_rec}
We performed a classification task using the Brodatz texture dataset~\cite{Brodatz_Dataset}.
Examples are shown in Fig.~\ref{fig:Texture_example}.
We followed the test protocol devised in~\cite{TSC_PAMI_2014}
and generated nine test scenarios with various number of classes.
This includes 5-texture (`5c', `5m', `5v', `5v2', `5v3'),
10-texture (`10', `10v') and 16-texture (`16c', `16v') mosaics.
To create a Riemannian manifold,
each image was first downsampled to {\small $256 \times 256$}
and then split into 64 regions of size {\small $32 \times 32$}.
The feature vector for any pixel {\small $I\left(x,y\right)$}
is $\Vec{f}(x, y)  =
  \Big(
    I\left(x,y\right),
    \left| \frac{\partial I}  {\partial x}  \right|,  \left|\frac{\partial I}  {\partial y}  \right|,
    \left| \frac{\partial^2 I}{\partial x^2}\right|,  \left|\frac{\partial^2 I}{\partial y^2}\right|
  \Big)^T$.
Each region is described by a {\small $5 \times 5$} covariance descriptor of these features.
For each test scenario, five covariance matrices per class were randomly selected as training data and the rest was used for testing.
The random selection of training/testing data was repeated 20 times.

Fig.~\ref{fig:Texture_results} compares the proposed RSR methods against
logEuc-SC~\cite{Guo_TIP13}
and TSC~\cite{TSC_PAMI_2014}.
In general, the proposed methods obtain the highest recognition accuracy on all test scenarios
except for the `5c' test, where both methods have slightly worse performance than TSC.
We note that in some cases such as  `5m' and `5v2', \mbox{RSR-J} performs better than \mbox{RSR-S}.
However, \mbox{RSR-S} is overall a slightly superior method for this task.

\begin{figure}[!tb]
       \centering
       \includegraphics[width=0.2 \columnwidth,keepaspectratio]{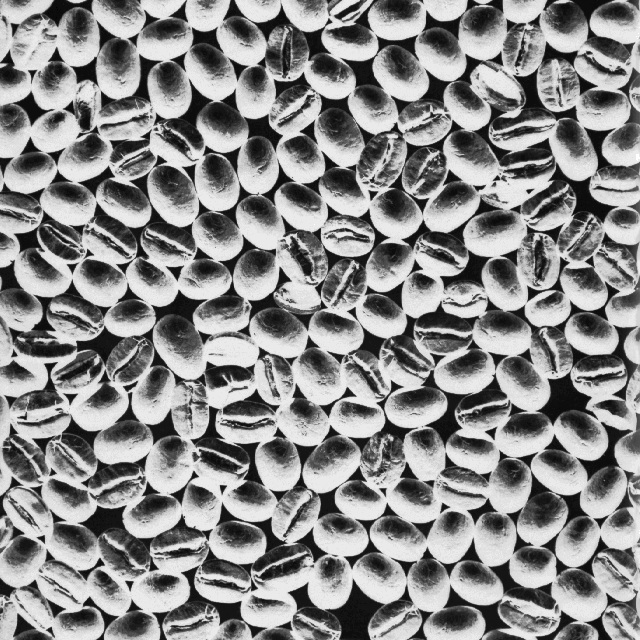}
       \includegraphics[width=0.2 \columnwidth,keepaspectratio]{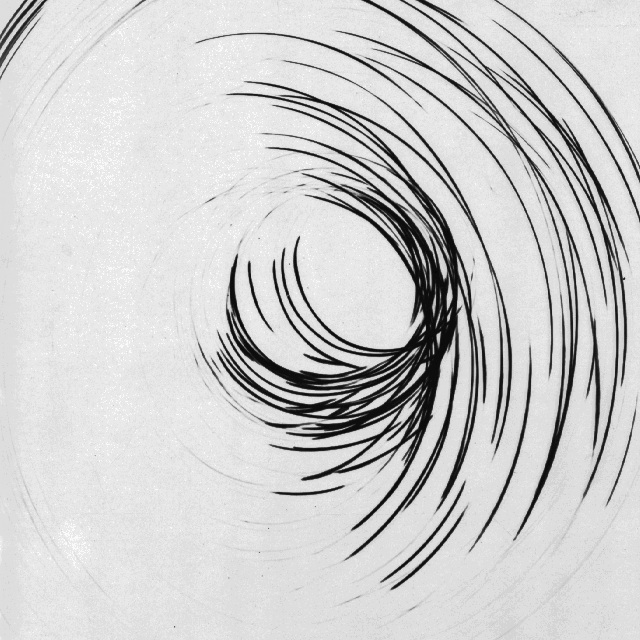}
       \includegraphics[width=0.2 \columnwidth,keepaspectratio]{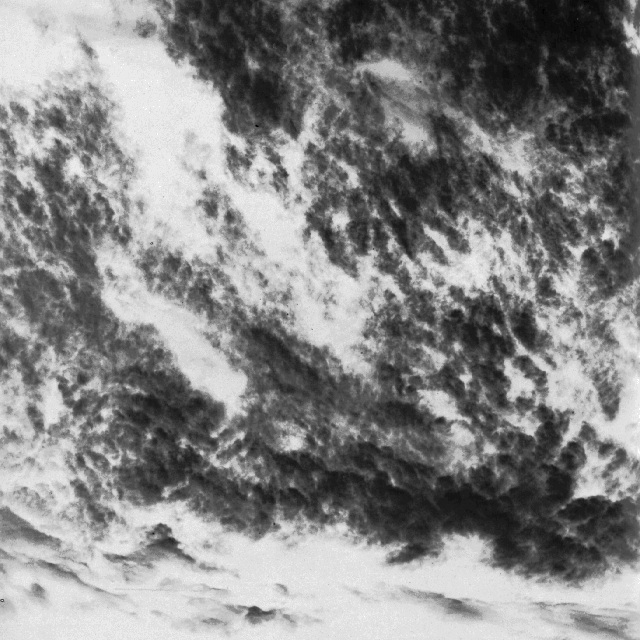}
       \caption{\small Examples from the Brodatz texture dataset~\cite{Brodatz_Dataset}.}
      \label{fig:Texture_example}
\end{figure}

\begin{figure}[!tb]
      \begin{minipage}{0.05\columnwidth}
        \rotatebox[origin=l]{90}{\scriptsize Average recognition accuracy}
      \end{minipage}
      \begin{minipage}{0.95\columnwidth}
        \centering
        \includegraphics[width=\columnwidth,keepaspectratio]{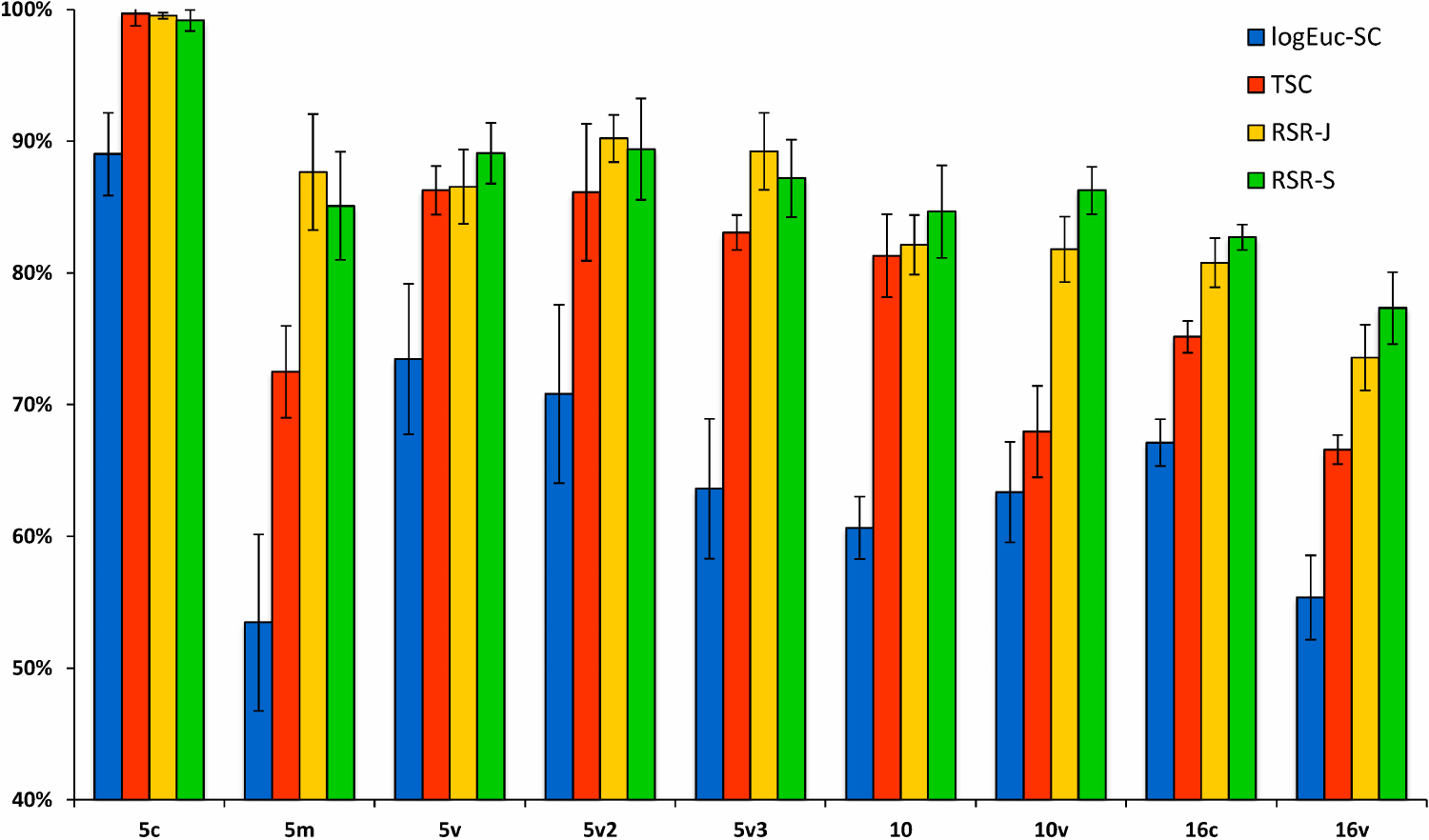}
        {\scriptsize Test ID}
      \end{minipage}
      \caption
        {
        \small
        Performance on the Brodatz texture dataset~\cite{Brodatz_Dataset} using
        log-Euclidean sparse representation (logEuc-SC)~\cite{Guo_TIP13},
        Tensor Sparse Coding (TSC)~\cite{TSC_PAMI_2014}
        and
        the proposed \mbox{RSR-J} and \mbox{RSR-S} approaches.
        The black bars indicate standard deviations.        
        }
    \label{fig:Texture_results}
\end{figure}

\subsection{Dictionary Learning}
\label{exp_dic_learning}
Here we analyse the performance of the proposed dictionary learning techniques as described in \textsection~\ref{sec:dic_learning} on two classification tasks:
texture classification and action recognition.

\subsubsection{Texture Classification}
\label{sec:exp_dic_texture}
Here we consider a multi-class classification problem,
using 111 texture images of the Brodatz texture dataset~\cite{Brodatz_Dataset}.
From each image we randomly extracted 50 blocks of size {\small $32 \times 32$}.
To train the dictionary, 20 blocks from each image were randomly selected,
resulting in a dictionary learning problem with 2200 samples.
From the remaining blocks, 20 per image were used as probe data and 10 as gallery samples.
The process of random block creation and dictionary generation was repeated twenty times.
The average recognition accuracies over probe data are reported here.
In the same manner as in~\textsection~\ref{sec:exp_texture_rec},
we used the feature vector
$\Vec{f}(x, y)  =
  \Big(
    I\left(x,y\right),
    \left| \frac{\partial I}  {\partial x}  \right|,  \left|\frac{\partial I}  {\partial y}  \right|,
    \left| \frac{\partial^2 I}{\partial x^2}\right|,  \left|\frac{\partial^2 I}{\partial y^2}\right|
  \Big)^T$
to create the covariance,
where the first dimension is the grayscale intensity,
and the remaining dimensions capture first and second order gradients.

We used the proposed methods to obtain the sparse codes,
coupled with a dictionary generated via two separate methods:
intrinsic {\it k}-means,
and the proposed learning algorithm~(\textsection~\ref{sec:dic_learning}).
The sparse codes were then classified using a nearest-neighbor classifier.

Figure~\ref{fig:dict_texture} shows the performance of \mbox{RSR-J} and \mbox{RSR-S} for various dictionary sizes.
Red curves show the performance when the intrinsic $k$-means algorithm was utilized for dictionary learning.
The blue curves demonstrate the recognition accuracies when the methods proposed in~\textsection~\ref{sec:dic_learning} were used for training,
and finally the green curves show the performance of log-Euclidean sparse coding equipped with K-SVD~\cite{Aharon_2006}
algorithm for dictionary learning.
The figures show that the proposed dictionary learning approach consistently outperforms $k$-means bar one case (\mbox{RSR-J} for dictionary size 8).
Using the proposed dictionary learning approach, \mbox{RSR-J} achieves the maximum recognition accuracy of $63.3\%$ with 104 atoms,
while \mbox{RSR-S} obtains the maximum accuracy of $60.6\%$ with 24 atoms.
In contrast, when intrinsic $k$-means is used for dictionary learning,
the maximum recognition accuracies for \mbox{RSR-J} and \mbox{RSR-S} are $52.9\%$ and $53.2\%$, respectively.
Furthermore, in all cases \mbox{RSR-S} is superior to the log-Euclidean solution, 
while \mbox{RSR-J} performs better than the log-Euclidean approach only for dictionaries with size larger than 24 atoms.

\def \DIC_SIZE_FIG {0.35}
\begin{figure*}[!tb]\centering
	\begin{subfigure}{\DIC_SIZE_FIG \textwidth}
		\includegraphics[width= 1.0 \columnwidth,keepaspectratio]{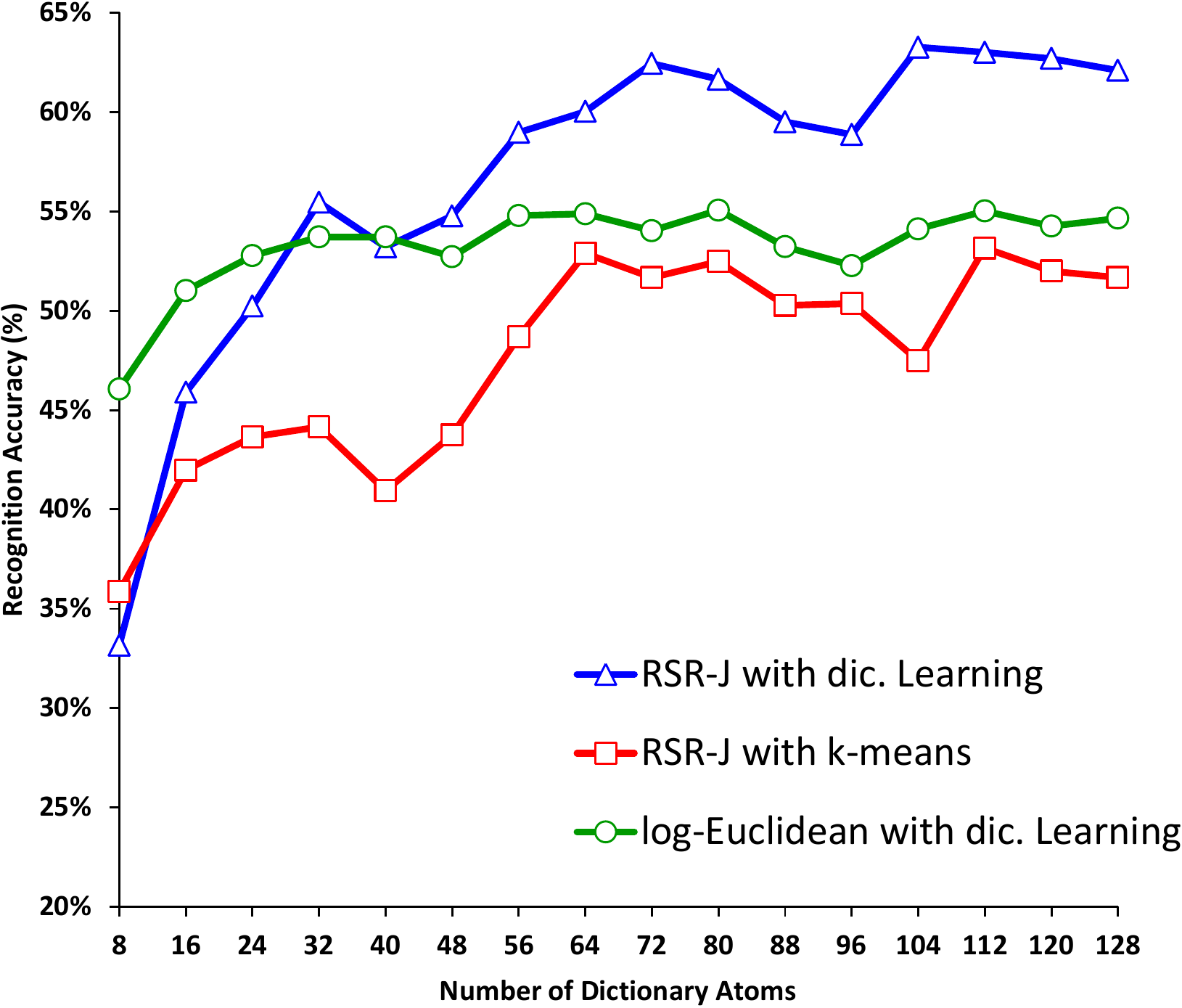}
		\caption{Sample images}
	\end{subfigure}	
	~~
	\begin{subfigure}{\DIC_SIZE_FIG \textwidth}
		\includegraphics[width = \columnwidth,keepaspectratio]{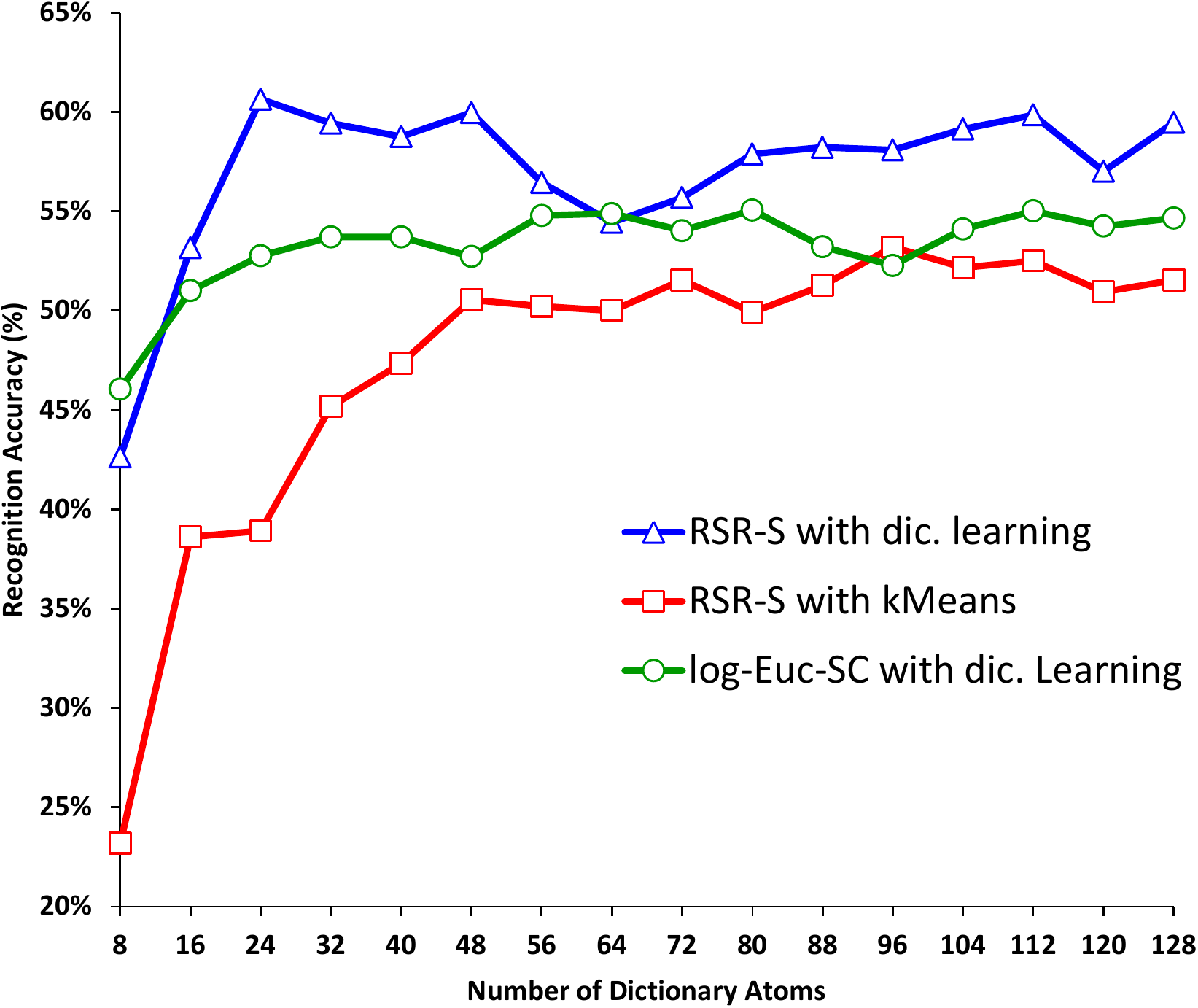}
		\caption{Seq.~\#1}
	\end{subfigure}
	\caption
        {\small
         Comparison of recognition accuracy versus size of dictionary for \mbox{RSR-J} and \mbox{RSR-S}. 
         The red curve shows the accuracy for dictionaries learned by intrinsic $k$-means algorithm. 
         The green curve shows the accuracy for dictionaries learned by log-Euclidean method, that is dictionary learning (K-SVD)
         along sparse coding on the identity tangent space. 
         The blue curve shows the accuracy for the proposed learning approach.
        }
    \label{fig:dict_texture}
\end{figure*}

\subsubsection{Action Recognition}
\label{sec:exp_dic_action}
The UCF sport action dataset~\cite{UCF_DATASET} consists of ten categories
of human actions including swinging on the pommel
horse, driving, kicking, lifting weights, running, skateboarding,
swinging at the high bar, swinging golf clubs,
and walking (examples of a diving action are shown in Fig.~\ref{fig:UCF_example}). The number of videos for each action varies
from 6 to 22 and there are 150 video sequences in total.
Furthermore, the videos presented in this dataset have non-uniform backgrounds and both the camera and the subject are moving in some actions.
Frames in all video sequences are cropped according to the region of interest provided with the dataset and then resized to {\small $64 \times 64$}.
The standard protocol in this dataset is the
leave-one-out (LOO) cross validation~\cite{UCF_DATASET,HDN_CVPR_2010,AFMLK_CVPR_2011}.

From each video, we extracted several RCMs by splitting video data into 3D volumes.
Volumes had the size of $32 \times 32 \times 15$ in $x-y-t$ domains with 8 pixels shift in each direction.
From each volume, a $12 \times 12$ RCM was extracted using kinematic features described in~\cite{Guo_TIP13}.
From training RCMs, we learned separate dictionaries for $J$ and $S$ divergences using the methods described 
in \textsection~\ref{sec:dic_learning} with 256 atoms each. The dictionaries were then used to determine the
video descriptor. To this end, each video was described by simply pooling the
sparse codes of its $32 \times 32 \times 15$ volumes using max operator. Having training and testing descriptors at our disposal,
a linear SVM~\cite{Bishop_2006} was used as classifier. 

\def\UCFSIZE {0.3}
\begin{figure}[!tb]
      \begin{minipage}{1.0\columnwidth}
      \center
        \includegraphics[width=\UCFSIZE\textwidth,keepaspectratio]{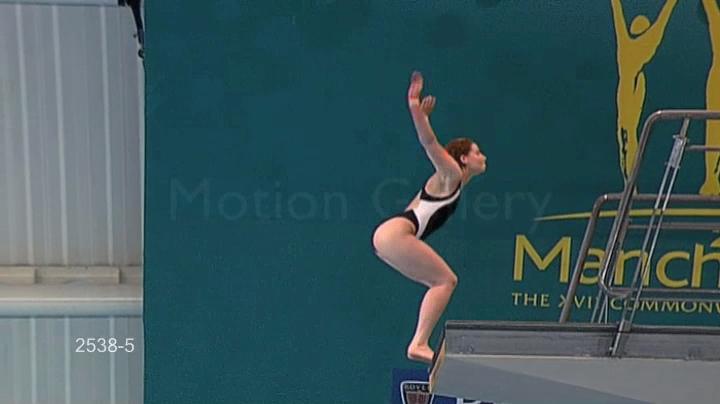}
        \includegraphics[width=\UCFSIZE\textwidth,keepaspectratio]{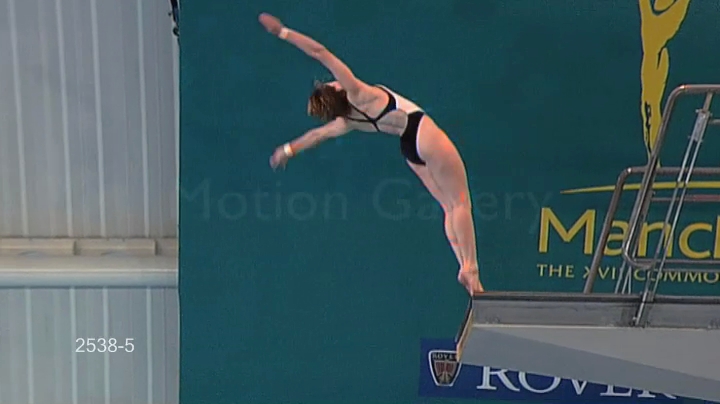}
        \includegraphics[width=\UCFSIZE\textwidth,keepaspectratio]{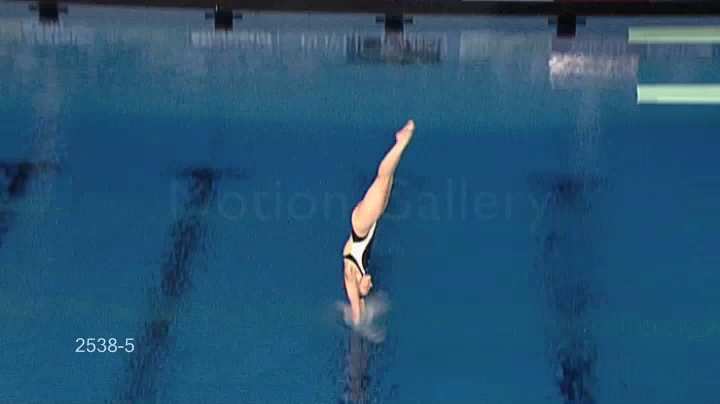}
      \end{minipage}
  \caption
    {
    \small
    Examples from the UCF sport action dataset~\cite{UCF_DATASET}.
    }
  \label{fig:UCF_example}
\end{figure}%
In Table~\ref{tab:table_ucf_performance},
the overall performance of the \mbox{RSR-J} and \mbox{RSR-S} methods is compared against three state-of-the-art Euclidean approaches:
HOG3D~\cite{HOG3D:2009},
Hierarchy of Discriminative space-time Neighbourhood features (HDN)~\cite{HDN_CVPR_2010},
and augmented features~\cite{AFMLK_CVPR_2011} in conjunction with multiple kernel learning (AFMKL).
HOG3D is an extension of histogram of oriented gradient descriptor~\cite{Dalal_Triggs_2005} to spatio-temporal spaces.
HDN learns shapes of space-time feature neighbourhoods that are most discriminative
for a given action category. The idea is to form new features composed of the neighbourhoods around the
interest points in a video.
AFMKL exploits appearance distribution features and spatio-temporal context features in a learning scheme for action recognition.
As shown in Table~\ref{tab:table_ucf_performance},
\mbox{RSR-J} outperforms the log-Euclidean approach and is marginally worse than AMFKL.
\mbox{RSR-S} achieves the highest overall accuracy.

\begin{table}[!tb] \small
\caption
    {
    Recognition accuracy (in \%) for the UCF action recognition dataset using
    HOG3D, HDN~\cite{HDN_CVPR_2010}, AFMKL~\cite{AFMLK_CVPR_2011}
    and the proposed \mbox{RSR-J} and \mbox{RSR-S} approaches.
    }
  \centering
    \begin{tabular}{lc} 
    \toprule
    {\bf Method}  &{\bf Recognition Accuracy} \\
    \toprule
    {\bf HOG3D~\cite{HOG3D:2009}}                               &$85.6$ \\
    {\bf HDN~\cite{HDN_CVPR_2010}}                              &$87.3$ \\
    {\bf AFMKL~\cite{AFMLK_CVPR_2011}}                          &$91.3$ \\
    {\bf \mbox{logEuc-SC with dic. learning}}          			&$89.3$ \\
    {\bf \mbox{RSR-J}}                          				&$90.7$ \\
    {\bf \mbox{RSR-S}}                          				&$\bf 94.0$ \\    
    \bottomrule
    \end{tabular}
    \label{tab:table_ucf_performance}
\end{table}

The confusion matrices for \mbox{RSR-J} and \mbox{RSR-S} divergences  are shown in
Tables~\ref{tab:ucf_confusion_matrix_j} and~\ref{tab:ucf_confusion_matrix_s}, respectively.
\mbox{RSR-J} perfectly classifies the actions of diving, golf swinging, kicking, riding horse, high bar swinging, walking and lifting, while 
\mbox{RSR-S} achieves perfect classification on golf swinging, riding horse, running, high bar swinging and lifting. Nevertheless, the overall 
performance of \mbox{RSR-S} surpasses that of \mbox{RSR-J} since \mbox{RSR-J} performs poorly in classifying the pommel-horse action.

\def \UCF_TBL_SIZE {0.3}
\begin{table*}[!tb]\scriptsize  
  \centering
  	\begin{minipage}{0.485 \textwidth}  \centering
  	\caption    { Confusion matrix (in \%) for the \textbf{\mbox{RSR-J}} method on the UCF sport action dataset using LOO protocol.}	
    \begin{tabular}{l m{\UCF_TBL_SIZE cm} m{\UCF_TBL_SIZE cm} m{\UCF_TBL_SIZE cm} m{\UCF_TBL_SIZE cm}
                      m{\UCF_TBL_SIZE cm} m{\UCF_TBL_SIZE cm} m{\UCF_TBL_SIZE cm} m{\UCF_TBL_SIZE cm} 
                      m{\UCF_TBL_SIZE cm} m{\UCF_TBL_SIZE cm}}
    \toprule
    {\bf}  &{\bf D} &{\bf GS} &{\bf K} &{\bf RH} &{\bf R} &{\bf S}
    &{\bf PH} &{\bf HS}    &{\bf W}  &{\bf L}\\
    \toprule
    {\bf D}                       &$\bf 100$   &$0$    &$0$    &$0$    &$0$    &$0$ &$0$    &$0$    &$0$    &$0$\\
    {\bf GS}                   &$0$   &$\bf 100$    &$0$    &$0$    &$0$    &$0$ &$0$    &$0$ &$0$    &$0$\\
    {\bf K}                  	   &$0$   &$0$    &$\bf 100$    &$0$    &$0$    &$0$ &$0$    &$0$ &$0$    &$0$\\
    {\bf RH}             	   &$0$   &$0$    &$0$    &$\bf 100$    &$0$    &$0$ &$0$    &$0$ &$0$    &$0$\\
    {\bf R}                      &$0$   &$0$    &$0$    &$0$    &$\bf 91.7$    &$0$   &$0$ &$0$ &$0$    &$8.3$\\
    {\bf S}                      &$0$   &$0$    &$23.1$  &$0$ &$0$ &$\bf 69.2$    &$0$   &$7.7$    &$0$  &$0$\\
    {\bf PH}                      &$0$    &$16.65$ &$16.65$    &$0$ &$0$   &$0$    &$\bf 25.0$    &$0$    &$0$    &$41.7$\\
    {\bf HS}                     &$0$    &$0$ &$0$    &$0$ &$0$   &$0$    &$0$    &$\bf 100$    &$0$    &$0$\\
    {\bf W}                      &$0$ &$0$    &$0$ &$0$    &$0$ &$0$    &$0$    &$0$    &$\bf 100$    &$0$\\
    {\bf L}                      &$0$    &$0$ &$0$    &$0$ &$0$   &$0$    &$0$    &$0$    &$0$    &$\bf 100$\\
    \bottomrule
    \label{tab:ucf_confusion_matrix_j}
    \end{tabular}    
  	\end{minipage}
    \quad
    \begin{minipage}{0.485 \textwidth}  \centering
  	\caption    { Confusion matrix (in \%) for the \textbf{\mbox{RSR-S}} method on the UCF sport action dataset using LOO protocol.}	    
	\begin{tabular}{l m{\UCF_TBL_SIZE cm} m{\UCF_TBL_SIZE cm} m{\UCF_TBL_SIZE cm} m{\UCF_TBL_SIZE cm}
                      m{\UCF_TBL_SIZE cm} m{\UCF_TBL_SIZE cm} m{\UCF_TBL_SIZE cm} m{\UCF_TBL_SIZE cm} 
                      m{\UCF_TBL_SIZE cm} m{\UCF_TBL_SIZE cm}}
    \toprule
    {\bf}  &{\bf D} &{\bf GS} &{\bf K} &{\bf RH} &{\bf R} &{\bf S}
    &{\bf PH} &{\bf HS}    &{\bf W}  &{\bf L}\\
    \toprule
    {\bf D}                       &$\bf 92.9$   &$0$    &$0$    &$0$    &$0$    &$0$ &$0$    &$7.1$    &$0$    &$0$\\
    {\bf GS}                   &$0$   &$\bf 100$    &$0$    &$0$    &$0$    &$0$ &$0$    &$0$ &$0$    &$0$\\
    {\bf K}                  	   &$0$   &$0$    &$\bf 95.0$    &$0$    &$15$    &$0$ &$0$    &$5.0$ &$0$    &$0$\\
    {\bf RH}             	   &$0$   &$0$    &$0$    &$\bf 100$    &$0$    &$0$ &$0$    &$0$ &$0$    &$0$\\
    {\bf R}                      &$0$   &$0$    &$0$    &$0$    &$\bf 100$    &$0$   &$0$ &$0$ &$0$    &$0$\\
    {\bf S}                      &$7.7$   &$0$    &$15.4$  &$0$ &$0$ &$\bf 69.2$    &$7.7$   &$0$    &$0$  &$0$\\
    {\bf P}                      &$0$    &$0$ &$0$    &$0$ &$0$   &$0$    &$\bf 83.3$    &$0$    &$0$    &$16.7$\\
    {\bf H}                     &$0$    &$0$ &$0$    &$0$ &$0$   &$0$    &$0$    &$\bf 100$    &$0$    &$0$\\
    {\bf W}                      &$7.7$ &$0$    &$0$ &$0$    &$0$ &$0$    &$0$    &$0$    &$\bf 92.3$    &$0$\\
    {\bf L}                      &$0$    &$0$ &$0$    &$0$ &$0$   &$0$    &$0$    &$0$    &$0$    &$\bf 100$\\
    \bottomrule
    \label{tab:ucf_confusion_matrix_s}
    \end{tabular}
    \end{minipage}    
\end{table*}

\section{Main Findings and Future Work}
\label{sec:future_work}

With the aim of addressing sparse representation on SPD manifolds,
we proposed to seek the solution through embedding the manifolds into RKHS
with the aid of two Bregman divergences, namely Stein and Jeffrey divergences.
This led to a relaxed and extended version of the Lasso problem~\cite{ELAD_SR_BOOK_2010} on SPD manifolds.

In Euclidean spaces, the success of many learning algorithms arises from their use of kernel methods~\cite{Shawe-Taylor2004book}.
Therefore, one could expect embedding a Riemannian manifold into higher dimensional Reproducing Kernel Hilbert Space (RKHS), where linear geometry applies, facilitates inference.
Such an embedding, however, requires a non-trivial kernel function defined on the manifold, which, 
according to Mercer's theorem~\cite{Shawe-Taylor2004book}, must be positive definite.
The approach introduced here attains its merit from the following facts:

\begin{itemize}

\item
By recasting the sparse coding from $\SPD{n}$ into RKHS, a convex problem is obtained which can be solved quite efficiently.
The sparse coding problem is in effect linearized, which is far easier than solving the Riemannian version of sparse coding as depicted in Eqn.~\eqref{eqn:Riemannian_sc}.
 
\item
Recasting the sparse coding from $\SPD{n}$ into RKHS exploits the advantages of higher dimensional Hilbert spaces,
such as easier separability of classes.

\item
The $J$ and $S$ divergences used in this paper are closely related to the Affine Invariant Riemannian Metric (AIRM)~\cite{Pennec_IJCV_2006},
and have several useful properties such as invariance to inversion and affine transforms.
However, unlike AIRM, the $J$ and $S$ divergences admit a Hilbert space embedding (\ie, can be converted to kernel functions).

\end{itemize}

Experiments on several classification tasks 
show that the proposed approaches achieve notable improvements in discrimination accuracy,
in comparison to state-of-the-art methods such as tensor sparse coding~\cite{TSC_PAMI_2014}.
We conjuncture that this stems from better exploitation of Riemannian geometry,
as both divergences enjoy several properties similar to affine invariant Riemannian metric on SPD manifolds.

We have furthermore proposed algorithms for learning a dictionary,
closely tied to the Stein and Jeffrey divergences.
The experiments show that in many cases better performance is achieved with \mbox{RSR-S} as compared to \mbox{RSR-J}.
However, we note that Jeffrey divergence enjoys several unique properties (\eg, closed form solution for averaging 
and Hilbert space embedding for all values of $\beta$ in Eqn.~\eqref{eqn:kernel_j_div}) which makes it attractive for analyzing SPD matrices. 
Future venues of exploration include devising other types of inference and machineries based on $J$ and $S$ divergences, such as structured learning.

\appendix
\subsection*{Geometric mean of J-Divergence}
\label{app:geometric_mean}

\begin{theorem}\label{thm:kld_airm_mean}
	For two matrices $\Mat{A}, \Mat{B} \in \SPD{n}$, the geometric mean of $J$ divergence $\Mat{A} \sharp_{J} \Mat{B}$
	and AIRM $\Mat{A} \sharp_{R} \Mat{B}$ are the same.	
\begin{proof}
For the $J$~divergence, we note that
\begin{small}
	\begin{align*}
		\frac{\partial\{J(\Mat{X},\Mat{A}) + J(\Mat{X},\Mat{B})\}}{\partial{\Mat{X}}} &= \frac{1}{2}
		\Big(\Mat{A}^{-1} + \Mat{B}^{-1}
		-\Mat{X}^{-1}(\Mat{A}+\Mat{B})\Mat{X}^{-1}\Big).
	\end{align*}
\end{small}
Therefore, $\Mat{A} \sharp_{J} \Mat{B}$ is the solution of 
\begin{align}
	\Mat{X}(\Mat{A}^{-1} + \Mat{B}^{-1})\Mat{X} = \Mat{A}+\Mat{B}, 
	\label{eqn:J_riccati_eqn}
\end{align}
which is a \textit{Riccati} equation with only one positive definite solution~\cite{BHATIA_2007}. We note that
\begin{align}
	\Mat{A} \sharp_{R} \Mat{B} &= \exp_{\Mat{B}}\Big(\frac{1}{2}\log_{\Mat{B}}(\Mat{A})\Big) = 
	\exp_{\Mat{A}}\Big(\frac{1}{2}\log_{\Mat{A}}(\Mat{B})\Big)\notag\\
	&= \Mat{A}^{\frac{1}{2}}\Big(\Mat{A}^{-\frac{1}{2}}\Mat{B}\Mat{A}^{-\frac{1}{2}}\Big)^{\frac{1}{2}}\Mat{A}^{\frac{1}{2}}
	. 
	\label{eqn:airm_mean_eqn}
\end{align}
It can be readily shown that $\Mat{A} \sharp_{R} \Mat{B}$ satisfies Eqn.~\eqref{eqn:J_riccati_eqn} which concludes the proof. 
\end{proof}	
\end{theorem}

\subsection*{Kernelized Feature Sign Algorithm}
\label{app:kernel_feature_sign_alg}

The efficiency of the Feature-Sign Search algorithm~\cite{NIPS2006_NG}
for finding sparse codes in vector spaces has been analysed in~\cite{lee2010:thesis}.
The algorithm was shown to outperform (in terms of speed and accuracy)
sparse solvers such as the generic QP solver~\cite{CVX}.
The gain is even higher for large and high-dimensional datasets (which are common in computer vision tasks).
As such, we have elected to recast the algorithm into its RKHS version,
in order to find the sparse codes on SPD manifolds.
We summarize the new version below, with the pseudo-code shown in Algorithm~\ref{alg:pseudocode_Feature_Sign_Search}. 

Given the objective function defined in Eqn.~\eqref{eqn:kernel_sparse_coding},
if the signs (positive, zero, or negative) of the $y_j$ are known at the optimal value,
each term of \mbox{$\left\|\Vec{y} \right\|_1=\sum\nolimits_{j=1}^{N}|y_i|$}
can be replaced by either {$y_i$}, {$-y_i$} or $0$.
Considering only non-zero coefficients,
this reduces Eqn.~\eqref{eqn:kernel_sparse_coding} to a standard,
unconstrained quadratic optimization problem (QP)~\cite{Boyd:2004},
which has an analytic solution.

The Feature-Sign Algorithm comprises of four basic steps.
The first two steps can be considered as a greedy search.
The search is directed towards selecting a new feature
that maximizes the rate of decrements in  Eqn.~\eqref{eqn:kernel_sparse_coding}.
This is accomplished by computing the first order derivative of Eqn.~\eqref{eqn:kernel_sparse_coding}
with respect to features, \ie,

\begin{align}
  	\frac{\partial}{\partial y_j} l_\phi(\Mat{X},\mathbb{D})
  	= \sum\limits_{q=1}^{N}{y_qk(\Mat{D}_q,\Mat{D}_j)}-2k(\Mat{X},\Mat{D}_j) + \lambda\;.
\end{align}

Since an estimate of the feature signs is available,
in the third step (feature-sign step) it is possible to find the solution of the unconstrained QP of
$\underset{\widehat{\Vec{y}}}{\min} \widetilde{f}(\widehat{\Vec{y}})$,
where

\begin{align}
    \widetilde{f}(\widehat{\Vec{y}})
   	&=
    \big\| \phi(\Mat{X}) -\hspace{-3ex} \sum\limits_{i \in \mathit{active\_set}}{\hspace{-3ex}y_i\phi(\Mat{D}_i)} \big\|^2
    + \lambda \widehat{\Vec{\theta}}^T \widehat{\Vec{y}} \nonumber \\
    	& =
    k(\Mat{X},\Mat{X})-2\widehat{\Vec{y}}^T\widehat{\Vec{\mathcal{K}}}
    +\widehat{\Vec{y}}^T \widehat{\Mat{\mathbb{K}}} \widehat{\Vec{y}}
    + \lambda \widehat{\Vec{\theta}}^T \widehat{\Vec{y}}\;.
    \label{eqn:Opt_Grass_sign}
\end{align}

In Eqn.~\eqref{eqn:Opt_Grass_sign},	$\widehat{\Vec{y}} = \Vec{y}(\mathit{active\_set})$,
$\widehat{\Vec{\theta}} = \Vec{\theta}(\mathit{active\_set})$
and	$\widehat{\Vec{\mathcal{K}}} = \Vec{\mathcal{K}}(\mathit{active\_set})$
are subvectors corresponding to  $\mathit{active\_set}$.
Similarly,	\mbox{ $\widehat{\Mat{\mathbb{K}}}=\Mat{\mathbb{K}}(\mathit{active\_set}, \mathit{active\_set})$}
is a submatrix corresponding to {$\mathit{active\_set}$}
with {$\mathit{active\_set}$} being the subset of selected features.
The closed form solution, \ie,
\mbox{$\widehat{\Mat{\mathbb{K}}}^{-1}(\widehat{\Vec{\mathcal{K}}}-\frac{\lambda \widehat{\Vec{\theta}}}{2})$},
can be derived by computing the root of first order derivative of Eqn.~\eqref{eqn:Opt_Grass_sign}
with respect to {$\widehat{\Vec{y}}$}.
The final step of the algorithm is an optimality check
to verify that the features and the corresponding signs
are truly consistent with the objective function in Eqn.~\eqref{eqn:kernel_sparse_coding}.

The convergence of the feature sign algorithm has been discussed for Euclidean spaces in~\cite{NIPS2006_NG}
and can be readily extended to the kernelized case.
To this end we first show that the feature-sign steps always strictly reduces the objective function 
in Eqn.~\eqref{eqn:kernel_sparse_coding_app}.

\begin{lemma}
	Each run of \emph{feature-sign step} in Algorithm~\ref{alg:pseudocode_Feature_Sign_Search}
	is guaranteed to strictly reduce the objective function.
	\label{lemma:convergence1}
\end{lemma}

\begin{proof}
	Let $\Vec{\widehat{y}}_c = [ ~\widehat{y}_{c,1}, ~\widehat{y}_{c,1}, ~\cdots, ~\widehat{y}_{c,N}]^T$
	be the current sparse codes corresponding to the given active set.
	We note that this step is invoked whenever optimality condition (a), described in Step~3, is not satisfied,
	\ie, there exists (at least) one feature in the current solution that is not consistent
	and $\Vec{\widehat{y}}_c$ is not an optimum point for the quadratic function depicted in Eqn.~\eqref{eqn:Opt_Grass_sign}.

	As a result, we have
	\mbox{$\widetilde{f}(\widehat{\Vec{y}}_\mathit{new}) < \widetilde{f}(\widehat{\Vec{y}}_{c})$}.
	If { $\widehat{\Vec{y}}_\mathit{new}$} is consistent with the given active set and sign vector,
	then updating {$\widehat{\Vec{y}} = \widehat{\Vec{y}}_\mathit{new}$} strictly decreases the objective.
	Otherwise, let {$\hat{y}_{\alpha}$} be the first zero-crossing point
	(where any coefficient changes its sign)
	on a line segment from {$\hat{\Vec{y}}_{c}$} to {$\hat{\Vec{y}}_\mathit{new}$}.
	Since $\widetilde{f}(\cdot)$ is convex,
	$\widetilde{f}(\widehat{\Vec{y}}_{\alpha}) < \widetilde{f}(\widehat{\Vec{y}}_{c})$.
	Therefore, the discrete line search described in Step 3 ensures a decrease in the objective function.
\end{proof}

\begin{lemma}
	Each run of \emph{feature evaluation} in Algorithm~\ref{alg:pseudocode_Feature_Sign_Search} 
	is guaranteed to strictly reduce the objective function.
	\label{lemma:convergence2}
\end{lemma}

\begin{proof}
	Let $y_{c,j}$ be the selected feature in step (2) of the algorithm. Such a feature exists since step (2) is invoked
	whenever the optimality condition (a) is not satisfied,
	\ie, there exists at least one zero coefficient ($y_{c,q}$) in $\Vec{\widehat{y}}_c$
	such that
	\mbox{$\left| \sum_{q=1}^{N}{y_qk(\Mat{D}_q,\Mat{D}_i)}-2k(\Mat{X},\Mat{D}_i) \right| > \lambda$}.
	Moreover, as the optimality condition (a) is satisfied,
	for all non-zero elements of \mbox{$y_{c,i} \in \Vec{\widehat{y}}_c$},
	we have
	\mbox{$\sum_{q=1}^{N}{y_qk(\Mat{D}_q,\Mat{D}_i)} - 2k(\Mat{X},\Mat{D}_i) + \lambda \operatorname{sign}(y_i) = 0$}.

	Let \mbox{$\Vec{\widehat{y}}_c=[~\widehat{y}_{c,1}, ~\widehat{y}_{c,1}, ~\cdots, ~\hat{y}_{c,N}]^T$}
	be the current sparse codes corresponding to the new active set (including $j$).
	We note that by adding {\small $y_{c,j}$} to the active set,
	the objective function {$\widetilde{f}(\cdot)$} is strictly decreased.
	This can be seen by considering a Taylor expansion of {$\tilde{f}(\cdot)$} around
	\mbox{$\Vec{\widehat{y}}=\Vec{\widehat{y}}_c$}.
	Because of optimality condition (a), the Taylor expansion has only one first order term in {$y_{c,j}$}
	and any direction that locally decreases {$\tilde{f}(\cdot)$}
	must be consistent with the sign of the activated {$y_{c,j}$}.
	Based on the previous lemma, and since the new solution has a lower objective value,
	the third step of feature sign algorithm ensures a decrease in the objective value.
\end{proof}

Having Lemmas~(\ref{lemma:convergence1}) and~(\ref{lemma:convergence2}) at our disposal,
we now show that a finite number of steps is required for convergence.
We first note that no solution can be repeated twice, as the objective value is strictly decreasing.
Moreover, the number of all possible active sets and coefficient signs is finite.
This ensures that the steps 2 to 4 of the algorithm cannot be repeated indefinitely.

\begin{algorithm}[!tb] \scriptsize
	\SetKw{KwGoTo}{go to}
	\SetKwBlock{KwInit}{Initialisation.}{}
	\SetKwBlock{KwProc}{Processing.}{}	
	\SetKwBlock{KwInvis}{}{}
	\setcounter{AlgoLine}{1}
	\LinesNotNumbered	
	\DontPrintSemicolon
	\KwIn
	{Query $\Mat{X} \in \SPD{n}$; the dictionary $\{ \Mat{D}_i \}_{i=1}^N,~\Mat{D}_i~\in~\SPD{n}$; 
	 Kernel function $k:\SPD{n} \times \SPD{n} \rightarrow \mathbb{R}$.	}
	\KwOut
	{The sparse code {$\Vec{y}$}.	}
	\BlankLine
	\KwInit{
	\mbox{$\mathit{active\_set} \gets \{\}$}, \mbox{$\Vec{y}_{N \times 1} \gets \Vec{0}$}, 
	\mbox{$\Vec{\theta}_{N \times 1} \gets \Vec{0}$}
	.\;
	}
	\KwProc{
		{\bf 1. Feature selection.}\;
		From the zero coefficients of {$\Vec{y}$}, 
		select \mbox{$i=\underset{i}{\arg \max}{ \left| 
		\sum\nolimits_{q=1}^{N}{y_qk(\Mat{D}_q,\Mat{D}_i)}-2k(\Mat{X},\Mat{D}_i) \right| }$}.\;
		{\bf 2. Feature evaluation.}\; 
		\If{$\left|\sum\nolimits_{q=1}^{N}{y_qk(\Mat{D}_q,\Mat{D}_i)}-2k(\Mat{X},\Mat{D}_i) \right| > \lambda$}
		{
		Add $i$ to the \mbox{$\mathit{active\_set}$}, \ie,
		\mbox{$\mathit{active\_set} \gets \mathit{active\_set} \cup{\{ i \}}$}\;
		Update the sign vector, \ie,
		\mbox{$\theta_i \mathrel{\mathop:}= -\operatorname{sign} 
		\left( \sum\nolimits_{q=1}^{N}{y_qk(\Mat{D}_q,\Mat{D}_i)}-2k(\Mat{X},\Mat{D}_i)-\lambda \right)$}
		}
		{\bf 3. Feature-sign step.}\;
		\KwInvis{
		$\widehat{\Vec{y}} \gets \Vec{y}(\mathit{active\_set})$,
		$\widehat{\Vec{\theta}} \gets \Vec{\theta}(\mathit{active\_set})$,
		$\widehat{\Vec{\mathcal{K}}} \gets \Vec{\mathcal{K}}(\mathit{active\_set})$\;		
		$\widehat{\Mat{\mathbb{K}}} \gets \Mat{\mathbb{K}}(\mathit{active\_set}, \mathit{active\_set})$\;	
		{$\widehat{\Vec{y}}_\mathit{new} \gets
		\widehat{\Mat{\mathbb{K}}}^{-1}(\widehat{\Vec{\mathcal{K}}}-\frac{\lambda}{2}\widehat{\Vec{\theta}})$},
		,\ie, the closed form solution to\;\Indp 
		\mbox{$\underset{\widehat{\Vec{y}}}{\min}
   		\left\| \phi(\Mat{X}) - \sum\limits_{i \in \mathit{active\_set}}{y_i\phi(\Mat{D}_i)} \right\|^2+
	    \lambda \widehat{\Vec{Y}} \widehat{\Vec{\theta}}$}.\;\Indm
	    \If{$\widehat{\Vec{y}}_\mathit{new}$ is consistent  with the current $\mathit{active\_set}$}
	    {
			$\widehat{\Vec{y}}=\widehat{\Vec{y}}_\mathit{new}$.
		}
		\Else(
		perform a discrete line search on the closed line segment from $\widehat{\Vec{y}}$ to $\widehat{\Vec{y}}_\mathit{new}$:
		)
		{
		Check the objective value at all points where any coefficient changes sign\;
		Update $\widehat{\Vec{y}}$ (and the corresponding entries in $\Vec{y}$) to the point with the lowest objective value.\;
		}
		Remove zero coefficients of $\widehat{\Vec{y}}$ from the $\mathit{active\_set}$\;
		\mbox{$\Vec{\theta} \gets \operatorname{sign}(\Vec{y})$}}		
		{\bf 4. Check optimality conditions.}\;
		\KwInvis{
		\If{
		$\exists  y_i \neq 0~~\rm{s.t.}~~
		\sum\nolimits_{q=1}^{N}{y_qk(\Mat{D}_q,\Mat{D}_i)}-2k(\Mat{X},\Mat{D}_i) + \lambda \operatorname{sign}(y_i)\neq 0$
		}
		{Jump to {\bf Feature selection.}}					 
		\ElseIf {
		$\exists y_i = 0~~\rm{s.t.}~~
		\left| \sum\nolimits_{q=1}^{N}{y_qk(\Mat{D}_q,\Mat{D}_i)}-2k(\Mat{X},\Mat{D}_i) \right| > \lambda$
		}
		{Jump to {\bf Feature-sign step.}}			
	}
	}
	\caption{\small Kernel feature-sign search, for finding sparse codes in reproducing kernel Hilbert spaces.}
	\label{alg:pseudocode_Feature_Sign_Search}
\end{algorithm}

\end{document}